\documentclass[final,11pt,noinfoline]{imsart}

\usepackage{booktabs}
\usepackage[utf8]{inputenc} 
\usepackage[T1]{fontenc}    
\usepackage{textcomp}
\usepackage{dsfont}
\usepackage{color}
\usepackage{xcolor}
\usepackage{graphics}
\usepackage{graphicx}
\usepackage{epstopdf}
\usepackage{amsmath,amsthm,amssymb,amsfonts}
\usepackage{mathrsfs}
\usepackage[garamont]{mathdesign}
\usepackage{enumitem}
\usepackage[colorinlistoftodos]{todonotes}
\usepackage[norelsize,ruled,vlined,commentsnumbered]{algorithm2e}
\usepackage{multirow}
\usepackage[final,activate,verbose=true,auto=true]{microtype}
\usepackage{a4wide}
\usepackage{bbm}
\usepackage{mathtools}
\usepackage{xcolor}
\usepackage{relsize}
 \usepackage{hyperref}
 \definecolor{burgundy}{rgb}{0.5, 0.0, 0.13}
\definecolor{camel}{rgb}{0.76, 0.6, 0.42}
\definecolor{chamoisee}{rgb}{0.63, 0.47, 0.35}
\definecolor{grey1}{RGB}{128,128,128}
 \hypersetup{colorlinks=true,linkcolor=burgundy,citecolor=chamoisee,urlcolor=burgundy,linktoc=page}
\usepackage{hhline}
\usepackage{makecell}
\newcommand{\eq}{\begin{equation}}
\newcommand{\qe}{\end{equation}}

\newcommand{\N}{\mathds{N}}                
\newcommand{\R}{\mathds{R}}                     

\def\B{\mathbf{B}}
\def\S{\mathds{S}}

\def\O{\mathcal{O}}
\def\R{\mathds{R}}
\def\N{\mathds{N}}
\def\S{\mathds{S}}
\def\E{\mathds{E}}
\def\P{\mathds{P}}
\def\H{\mathcal{H}}

\def\V{\mathcal{V}}
\def\Y{\mathcal{Y}}
\def\F{\mathcal{F}}

\def\I{\operatorname{Id}}

\def\O{\mathcal{O}}
\def\R{\mathds{R}}
\def\B{\mathds{B}}
\def\Sp{\mathds{S}}
\def\sp{\operatorname{span}}
\def\N{\mathds{N}}
\def\P{\mathds{P}}
\def\E{\mathds{E}}

\def\V{\mathcal{V}}
\def\I{\operatorname{Id}}

\def\dim{\operatorname{dim}}

\usepackage{mathtools}
\usepackage{tikz}
\usetikzlibrary{arrows, decorations.markings, calc, fadings, decorations.pathreplacing, patterns, decorations.pathmorphing, positioning}

\newtheorem{theorem}{Theorem}
\newtheorem{lemma}[theorem]{Lemma}

\newtheorem{proposition}[theorem]{Proposition}
\newtheorem{defn}[theorem]{Definition}
\newtheorem{remark}{Remark}

\date{\today}

\begin{document}
\sloppy

\begin{frontmatter}

\title{Random Geometric Graphs on Euclidean Balls}
\runtitle{RGG on Balls}

\begin{aug}
\author{\fnms{Ernesto} \snm{Araya Valdivia}\ead[label=e1]{ernesto.araya-valdivia@math.u-psud.fr}}
\affiliation{Universit\'e Paris-Saclay}
\address{Laboratoire de Math\'ematiques d'Orsay (LMO) \\ Universit\'e Paris-Saclay, 91405 Orsay Cedex, France}
\runauthor{Araya}
\end{aug}


\begin{abstract}
We consider a latent space model for random graphs where a node $i$ is associated to a random latent point $X_i$ on the Euclidean unit ball. The probability that an edge exists between two nodes is determined by a ``link'' function, which corresponds to a dot product kernel. For a given class $\F$ of spherically symmetric distributions for $X_i$, we consider two estimation problems: latent norm recovery and latent Gram matrix estimation. 
We construct an estimator for the latent norms based on the degree of the nodes of an observed graph in the case of the model where the edge probability is given by $f(\langle X_i,X_j\rangle)=\mathbbm{1}_{\langle X_i,X_j\rangle\geq \tau}$, where $0<\tau<1$. We introduce an estimator for the Gram matrix based on the eigenvectors of observed graph and we establish Frobenius type guarantee for the error, provided that the link function is sufficiently regular in the Sobolev sense and that a spectral-gap-type condition holds. We prove that for certain link functions, the model considered here generates graphs with degree distribution that have tails with a power-law-type distribution, which can be seen as an advantage of the model presented here with respect to the classic Random Geometric Graph model on the Euclidean sphere. We illustrate our results with numerical experiments.
\end{abstract}




\end{frontmatter}

\maketitle 

\section{Introduction}

Given the ubiquity of network structured databases, the task of extracting information from them has become an important topic within many scientific communities, including statistics and machine learning. This has gone in hand with the development, mainly in the last decade, of powerful tools of graph theory, such as the graphon theory \cite{Lovsze,Lova1,Lova3}, which describes the asymptotic behavior of large dense graphs.

In this paper we will focus on extracting information from a single observation of a graph, which we assume generated from a parametric family of models, with latent space structure, which we will call random geometric graphs (RGG) on the Euclidean ball $\B^d=\{x\in\R^d:\|x\|\leq 1\}$. The model we will consider here has similarities not only with the random geometric graph model on the sphere and its generalizations, considered for example in \cite{Bubeck16, Yohann}, but also with the random dot product graph model (RDPG) \cite{RDPG,Sussman}. Indeed, one of our goals is to show that the random graph model presented here is flexible enough to generate graphs that have a degree profile distributed according to a power-law type distribution, while maintain some of the structural qualities that make it well-suited for statistical inference. 

We will consider a particular instance of the $W$-random graph model for dense graphs \cite[Ch.10]{Lov}, where a kernel function $W$ defines the probability of connection between two latent points. Similar to the context treated in \cite{Ara}, we will consider $W$ to be a dot product kernel, but here the ambient space will be the Euclidean ball, instead of the unit sphere. More specifically, we will consider that each node of a graph is associated with a randomly placed latent point in $\B^d$ (in an i.i.d manner) according to a probability distribution that belongs to a parametric family of spherically symmetric distribution, that we will call $\F$.  The main difference with the spherical case is that when considering $\B^d$ as the ambient space it is not only the angle between the latent points which determines the probability of connection between two nodes, but also their norm. This offer more flexibility in the degree distribution of the generated graphs, which in the spherical case is concentrated around a single value, at least when the latent points are distribution according to the uniform distribution (which is the only spherically symmetric probability measure). In particular, for certain link functions, we will show that the degree sequence exhibits a power-law type distribution. To best of our knowledge, there is no standard definition of power law type distribution in the graphon literature. We will introduce a notion of power-law distributions in this context based on the (normalized) degree function $d_W(\cdot)$ defined on graphon \cite[Sec.7.1]{Lov}.   

We discuss two problems of estimation of latent information on this model. We first study possibility of estimation of the latent norm from the observed adjacency matrix, in the threshold graphon model, that is when the link function (or graphon) is of the form $f(\langle x,y \rangle)=\mathbbm{1}_{\langle x,y\rangle\geq \tau}$, for a $\tau>0$ and $x,y\in\B^d$. In this model, two nodes will be connected if their latent points have inner product larger than $\tau$. We propose an estimator for the norm of the latent points based on the degree of the correspondent latent point. We prove the consistency of the estimator and illustrate its performance by simulations.

We next study the problem of estimating the Gram matrix of the latent points for the RGG model on $\B^d$, proposing an estimator which is based on a set of eigenvectors of the observed adjacency matrix, which extends the spectral approach developed in the spherical case \cite{Ara}. Our main assumption is related to the spectral gap between certain eigenvalues of the integral operator associated to the link function. This type of assumptions have been used before in the literature, mainly in the context of matrix estimation and manifold learning, often because some version of the Davis-Kahan $\sin {\theta}$ theorem is used as a technical step for proving finite sample bounds on the eigenvectors (see for instance~\cite{Chatterjee, ManLearning, TangSuss}). We will prove finite sample guarantees for the Frobenius error of our estimator, under the spectral gap assumption. In particular, we will prove that under certain Sobolev regularity assumptions the rate of convergence for the proposed Gram matrix estimator will be parametric. Hence, the results presented here not only extend the approach developed in \cite{Ara}, but also improve the convergence rate. The proof will be mainly based on the harmonic analysis on $\B^d$ and matrix concentration inequalities for the operator norm\cite{Tropp,Vershy, BanVan}.  

It is worth mention that some related problems, involving the recovery of latent structures, have recently been studied in \cite{Athreya}, from the spectral point of view, but on the RDPG model and with distributional assumptions of the latent points and ambient spaces different from the ones we consider here.  

\subsection{Notation}

We will use the asymptotic notation as usual. For a real function $f$, we write $f(x)=\O(g(x))$, for $g$ strictly positive, if and only if there exists $C>0$ such that $|f(x)|\leq C g(x)$ for $x$ larger that certain $x_0$. We use the symbol $\lesssim$ to denote inequality up to constants, that is $f(x)\lesssim g(x)$ if and only if there exist $C>0$ such that $f(x)\leq C g(x)$. Similarly, $f(x)\lesssim_\alpha g(x)$ will denote $f(x)\leq C(\alpha) g(x)$, that is the constant might depend on $\alpha\in \R$. We use $I(x;a,b)$ to denote the regularized incomplete Beta function, that is $I(x;a,b)~=~\frac{1}{B(a,b)}\int^x_0t^{a-1}(1-t)^{b-1}$, where $B(a,b)=\frac{\Gamma(a)\Gamma(b)}{\Gamma(a+b)}$. We will use $I_x(a,b)$ and $I(x;a,b)$ indistinctly. $I^{-1}(x;a,b)$ will denote the inverse of $I(x;a,b)$.

\section{Random Geometric Graphs on Balls}\label{sec:RGG_ball}
In this section we describe the RGG model on the Euclidean ball. We will restrict ourselves to a set of measures with spherical symmetry for the latent points distribution. The reason is mainly technical and is related with the framework of harmonic analysis on $\B^d$, which is one of the main ingredients in our approach for estimate the latent distances. Part of the material presented here is classic in the context of harmonic analysis on $\B^d$ \cite[Chap.11]{Dai}, including the geometric formulas on Euclidean 
spaces with measures defined by Jacobi weights. 

We define $\mathcal{F}=\{F_\nu\}_{\nu>-1/2 }$ the parametric family of distributions on $\B^d$ with densities, with respect to the Lebesgue measure, given by \[dF_\nu(x)=C_\nu(1-\|x\|^2)^{\nu-\frac12}\]
where $C_\nu=\int_{\B^d}(1-\|x\|^2)^{\nu-\frac12}dx$. Observe that for $\nu=\frac12$ the distribution $F_\nu$ is equal to the uniform distribution on $\B^d$.

From the expression for $dF_\nu(x)$ we can deduce the distribution of the norm of the latent points. More specifically, if $X$ is a $\B^d$-valued random variable distributed according to $F_\nu$, then $\|X\|^2$ follows a distribution $Beta(\frac d2,\nu+\frac12)$ (see Lemma \ref{lem:dist_norm}). 


In this context, the generative model based on the $W$-random graph model \cite[Chap.10]{Lov} is described by a two step procedure as follows: given $F_\nu\in\mathcal{F}$ and a link function $f:[-1,1]\rightarrow [0,1]$, which we assume measurable, we first sample i.i.d latent points $\{X_i\}_{1\leq i\leq n}$ according to $F_\nu\in\mathcal{F}$, for some $\nu>0$. Then, conditional to these latent points, we sample the adjacency matrix $A_{ij}$ such that for $i<j$, the entries $A_{ij}$ are independent Bernoulli variables and 
\[\P(A_{ij}=1)=f(\langle X_i,X_j \rangle)\] The entries $A_{ij}$ for $i>j$ are defined by symmetry and recall that $A_{ii}=0$, for all $i\in[n]$. 
This model contains as subclasses some classic random graphs models such as the Erd\"os-R\'enyi model, where $f(t)=p$ for $p\in [0,1]$, threshold or proximity graphon $f(t)=\mathbbm{1}_{t\geq \tau}$ for $\tau\geq 0$ and the random dot product graph for $f(t)=\frac12(1-t)$.

\subsection{The degree function}
We recall the definition of the graphon degree function \cite{Lov}[Chap. 7], which can be seen as analogous to the normalized degree of a node on a finite graph. Let $W$ be a graphon defined on $\Omega$ with measure $\mu$, the (normalized) degree function is defined as follows 
\[d_W(x):=\int_{\Omega}W(x,y)d\mu(y)\]
In the case of the Erd\"os-R\'enyi model, that is when $W(x,y)=p$, for some $p\in[0,1]$, we have that $d_W(x)=p,\ \forall x\in\Omega$ (which is valid for any measurable space $(\Omega,\mu)$). In the case of a graphon of the form $W(x,y)=f(\langle x,y\rangle)$ defined on $\Omega=\S^{d-1}$ with $\mu=\sigma$, where $\sigma$ is the uniform measure on $\S^{d-1}$, we have that $d_W(x)$ is also constant (this follows by a simple change of variables). When $W(x,y)=f(\langle x,y\rangle)$ and $\Omega=\B^d$ and $\mu=F_{\nu}$, for some $\nu>-\frac12$, we see that the $d_W(x)=~d_W(x')$ for all $x,x'\in\B^d$, such that $\|x\|=\|x'\|$. Take for instance the threshold function $W_g(x,y)=\mathbbm{1}_{\langle x,y\rangle\geq \tau}$, for some $\mu=F_\nu$ \footnote{With some abuse of notation we use $F_\nu$ for the distribution function and the measure.}. Then we have for the degree function
\begin{align*}
d_{W_g}(x)&=\int_{\B^d}\mathbbm{1}_{\langle x,y\rangle\geq \tau}dF_\nu(x)\\
&=F_\nu\Big(\operatorname{Sc}\big(x,1-\frac{\tau}{\|x\|\vee\tau}\big)\Big)
\end{align*}
where $\operatorname{Sc}(x,h)$ represents the spherical cap on $x/\|x\|$ with heigh $h$, that is \[\operatorname{Sc}(x,h):=\{y\in\B^d: \langle y,x/\|x\|\rangle\geq 1-h\}\]
Fix $X_i\in\B^d$, then the probability that $X_j$ is connected to $X_i$ for $j\neq i$ is \[\P(A_{ij}=1)=F_\nu\Big(\operatorname{Sc}\big(X_i,1-\frac{\tau}{\|X_i\|\vee\tau}\big)\Big)\]
Note that if $\|X_i\|\leq \tau$, then the spherical cap in the previous formula reduce to a point and, therefore, has measure zero. In other words, the points $X_i$ such that $\|X_i\|< \tau$ are associated with isolated nodes and the points such that $\|X_i\|=\tau$ are almost surely isolated. The degree function on a graphon can be regarded as the continuous analog to the normalized degree on a finite graph. To make this more precise, we denote $d_G(X_i):=\sum_{j\neq i}A_{ij}$ the degree of $X_i$ in the random graph. Observe that the random variable $d_G(X_i)$, conditional to $X_i$, follows a distribution $Binomial(n-1,d_W(X_i))$, thus \begin{equation}\label{eq:meandeg}\E\Big(\frac{d_G(X_i)}{n-1}\Big)=d_{W_g}(X_i)=F_\nu(\operatorname{Sc}(X_i,1-\tau/\|X_i\|))\end{equation}

We will use the degree of an observed graph (from the RGG on $\B^d$ model) to deduce the latent norm in the following way. First, from standard concentration inequalities, we deduce for each $i$ the degree is highly concentrated around its mean $d_{W_g}(X_i)$. From the spherical symmetry of $F_\nu$, we deduce that for each $i$, the right hand side of \eqref{eq:meandeg} depends only on $\|X_i\|$ and from Lemma \ref{lem:degree_fnct} we deduce the explicit form of the relation (and its inverse) that maps $\|X_i\|$ into $d_{W_g}(X_i)$. From this, we define an estimator of $\|X_i\|$ based on the degree of the node $i$ and proves its consistency (in Proposition \ref{prop:conv_est_norm}).

We will also use the degree function to prove that for certain link functions, the degree sequence presents tails that decay as a power law. As pointed in \cite[Sec.9]{Jan}, there is no standard definition of graph with power-law distributed degrees. While power-law for the degree are mentioned in the graphon literature, such as \cite{Chayes2}, no precise definition is given. We will introduce the following definition

\begin{defn}\label{defn:power-law}
Given a graphon $W$, defined on $(\Omega,\mu)$, we will say that its degree has power law tails if there exist $0\leq\kappa<1$, $C>0$ and $\theta>0$ such that \[\mu(\{x\in\Omega: d_W(x)\geq h\})\geq C(h-\kappa)^{-\theta}\] for $\kappa<h<1$.
\end{defn} 

\begin{remark}
In the case $\kappa>0$, it can be said that the degree is distributed according to a shifted power-law. 
\end{remark}

In Section \ref{sec:power law} we will study the power law property for some specific graphons on $\B^d$. 

\section{Main results}\label{sec:main_results}

Here we gather the main results of this paper. We start by considering the case of graphs generated by the threshold graphon with parameter $\tau>0$, that it $W_g(x,y)=\mathbbm{1}_{\langle x,y\rangle\geq\tau}$. Given the observation of graph of size $n$, which we suppose generated by the $W$-random graph model with $W_g$ and latent points distributed according to some $F_\nu\in\mathcal{F}$, we want to obtain information about the latent points $X_i$. In particular, we are interested in estimating the norm of latent points $\{\|X_i\|\}^n_{i=1}$ and the Gram matrix $\mathcal{G}^*$, which has entries $\mathcal{G}^*_{ij}=\langle X_i,X_j\rangle$.

It is not possible to estimate the latent norms, or any positional information for that matter, if no restriction on the measure is imposed. That is, there exist different combinations of measures in $\mathcal{F}$ and link functions that generate the same random graph model (the same distribution over finite graphs). This is given in Proposition \ref{prop:non_inden}, below.

We consider the threshold graphon model on $\B^d$, defined by the kernel $W(\langle x,y \rangle)=\mathbbm{1}_{\langle x,y \rangle}\geq \tau$ for a fixed $F_\nu\in\mathcal{F}$ and $\tau>0$. We assume that a graph $G$ is observed from that model. For each node $i$ in $G$ we define \[Z_i:=I^{-1}(2\frac{d_G(i)}{n-1}; \nu+\frac d2,\frac12)\] For a fixed $i$, we define the following estimator of the norm $\|X_i\|$ based on $Z_i$\[\hat{\zeta}_i:=\frac{\tau}{\sqrt{1-Z_i}}\vee 1\]

The following proves the consistency of the estimator and it is mainly a consequence of the strong law of large numbers.
\begin{proposition}\label{prop:conv_est_norm}
For a fixed $i\in\N$, the random variable $\hat{\zeta}_i$ converges almost surely to $\|X_i\|$. 
\end{proposition}


We now turn our attention to the problem of estimating the Gram matrix of the latent points. We will consider $\mathcal{G}^*=\frac 1n(1-\delta_{ij}) \langle X_i,X_j\rangle$ the population Gram matrix (with the diagonal erased) and $\mathcal{G}_U:=\frac{1}{2\tilde{c}_\nu(1+\gamma_\nu)}UU^T$ for any $n\times d$ real matrix $U$. The reason for the chosen normalization $2\tilde{c}_\nu(1+\gamma_\nu)$, comes from the harmonic analysis on $\B^d$ and it will be clarified in Section \ref{sec:eigensystem} below. 
We now give a slightly informal version of our main result regarding the estimation of $\mathcal{G}^*$, which will be stated more formally in Section \ref{sec:eigap}. 
\begin{theorem}[Informal version]\label{thm:main3_informal}
Let $W$ be a graphon defined by a dot product kernel on $\B^d$ and measure $F_\nu\in\F$. If $W$ is sufficiently regular, in the Sobolev sense, and satisfy a spectral gap condition, then there exists a set of $d$ eigenvalues $\hat{v}_1,\cdots,\hat{v}_d$ of the normalized adjacency matrix of the observed graph, such that with high probabilty
\[\|\mathcal{G}^*-\hat{\mathcal{G}}\|_F\leq C(W)\frac{1}{\sqrt n}\]
where $\hat{\mathcal{G}}=\mathcal{G}_{\hat{V}}$ and $\hat{V}$ is the matrix with columns $\hat{v}_1,\hat{v}_2,\cdots, \hat{v}_d$. 
\end{theorem}
In the previous theorem the constant $C(W)$ will depend on the spectral gap of $W$, which will be defined in Section \ref{sec:eigap} .


We will now see the main elements for the proof of Theorem \ref{thm:main3_informal}, which are related to the properties of the spectrum of the integral operator $T_W$. 

\section{Graphon eigensystem and harmonic analysis on $\B^d$}\label{sec:eigensystem}

We will extend the method developed in \cite{Ara} for the case of geometric graphs on the sphere. For that, it will useful to consider an orthogonal polynomials basis of $L^2(\B^d,F_\nu)$, with respect to the inner product given by \[\langle f,g\rangle=a_\nu\int_{\B^d}f(x)g(x)dF_\nu(x)dx\] where we recall that $dF_\nu(x)=(1-\|x\|^2)^{\nu-\frac12}$ and $a_\nu=1/\int_{\B^d}F_{\nu}(x)dx$.

We denote $\mathcal{Y}_n$ the subspace of orthogonal polynomials on $d$ variables (with respect to the inner product defined above) of degree exactly $n$. It is implicit that $\mathcal{Y}_n$ depend on $\nu$. From \cite[p.266]{Dai} we know that the space dimension is $\dim{(\mathcal{Y}_n)}=\binom{n+d-1}{n}$ (this actually can be seen by applying a Gram-Schmidt orthonormalization process to monomials). 


There are explicit expressions (closed formulas) for the reproducing kernel $P^\nu_n(x,y)$ of each $\mathcal{Y}_n$. We recall that $P^\nu_n(x,y)$ is said to satisfy the reproducing property on $\mathcal{Y}_n$ if and only if \[p(x)=\int_{\B^d}P^\nu_n(x,y)p(y)dF_{\nu}(y),\quad \forall p\in \mathcal{Y}_n,\ \forall x\in\B^d\] In our context, the reproducing kernel $P^\nu_n(x,y)$ is the projector of $L^2(\B^d,F_\nu)$ onto $\mathcal{Y}_n$, which can be seen by writing $P^\nu_n(x,y)=\sum_{p_{k,n}\in Q_n}p_{k,n}(x)p_{k,n}(y)$, where $Q_n:=\{p_{k,n}\}^{\dim{\Y_n}}_{k=1}$ is an $L^2(\B^d,dF_\nu)$-orthonormal basis of $\mathcal{Y}_n$. A key result is the following close form representation of the reproducing kernel \cite[Cor.11.1.8]{Dai} (see also \cite[Eq.2.2]{XuBall}), for $\nu>0$ \begin{equation}\label{eq:RKn1}
P^\nu_n(x,y)=c_\nu\frac{n+\gamma_\nu}{\gamma_\nu}\int_{-1}^1G_n^{\gamma_\nu}(\langle x,y\rangle+\sqrt{1-\|x\|^2}\sqrt{1-\|y\|^2}t)(1-t^2)^{\nu-1}dt\end{equation}
where $\gamma_\nu:=\nu+\frac{d-1}{2}$, $c_\nu=\frac{\Gamma(\nu+1/2)}{\sqrt{\pi}\Gamma(\nu)}$ and $G_n^{\gamma_\nu}(\cdot)$ is the Gegenbauer polynomial of degree $n$ with weight $\gamma_\nu$. 

It is well known that $\{G_n^{\gamma_\nu}(\cdot)\}_{n\geq 0}$ forms a basis for $L^2([-1,1],\gamma_\nu)$ \cite{Sze}. 
In addition, each 
$p_k\in\mathcal{Y}_n$ is an eigenfunction of the following $L^2(\mathds{B}^d,dF_{\nu})$ integral operator 
\[T_f g(x)=\int_{\mathds{B}^d}f(\langle x,y \rangle)g(y)dF_\nu(y)dy\] for any $f\in L^1([-1,1],\gamma_\nu)$ (which is automatic in our case, given that $f$ is bounded). The previous statement is a consequence of the Funk-Hecke formula in this context \cite[Thm.11.1.9]{Dai}, which also give us a formula for the eigenvalue associated to each $p_k\in \mathcal{Y}_n$ is 
\begin{equation}\label{eq:eigs}
\lambda^\ast_n(f)=c_{\gamma_\nu}\int_{-1}^1f(t)\frac{G^{\gamma_\nu}_n(t)}{G^{\gamma_\nu}_n(1)}(1-t^2)^{\gamma_\nu-1/2}dt
\end{equation}
and $c_{\gamma_\nu}$ is such that $\lambda^\ast_0(1)=1$. Notice that for each $n\in\mathds{N}$ we have the following decomposition of the reproducing kernel of $\mathcal{Y}_n$ in terms of the basis elements $p_{k,n}$ \begin{equation}\label{eq:RKn2}
P^\nu_n(x,y)=\sum_{p_k\in Q_n}p_{k,n}(x)p_{k,n}(y),
\end{equation}
hence for a given $f(\langle x,y \rangle)$ the following decomposition holds, in virtue of the spectral theorem for compact operators  \begin{equation}
\label{eq:decompF}f(\langle x,y\rangle)=\sum_{n\in\mathds{N}}\lambda^\ast_n(f)P^\nu_n(x,y)\end{equation}

We note that previous implies that $\lambda_n^*(f)$ is an eigenvalue associated to every $p_{k,n}\in Q_n$, which means that it has multiplicity $\dim{(\Y_n)}$. 
We come back to formula \eqref{eq:RKn1}, which in the linear case gives a representation of the reproducing kernel $P^\nu_1$ in terms of the inner product $\langle x,y\rangle$. Given that $G^\gamma_1(t)=2\gamma t$, we have 
\begin{align}
    P^\nu_1(x,y)&=2 c_\nu \frac{1+\gamma_\nu}{\gamma_\nu}\int^1_{-1}(\langle x,y\rangle+\sqrt{1-\|x\|^2}\sqrt{1-\|y\|^2}t)(1-t^2)^{\nu-1}dt\nonumber\\
    &=2\tilde{c_\nu}(1+\gamma_\nu)\langle x,y\rangle
\end{align} 
where $\tilde{c_\nu}=c_\nu\int^1_{-1}(1-t^2)^{\nu-1}dt$. In the last step we used that $\int^1_{-1}t(1-t^2)^{\nu-1}dt=0$ given the parity of the function inside the integral. From formula \eqref{eq:RKn2} we deduce that 
\begin{equation}\label{eq:rec_form}
\frac{1}{2\tilde{c_\nu}(1+\gamma_\nu)}\sum_{p_{k,n}\in Q_n}p_{k,n}(x)p_{k,n}(y)=\langle x,y\rangle
\end{equation}

The previous relation has its analogous in the case of dot product kernels on  $\Sp^{d-1}$, see \cite{Ara}. We can read formula \eqref{eq:rec_form} as a presentation for the Gram matrix of the latent points in terms of the elements of the orthonormal basis of $\mathcal{Y}_1$ (which has exactly $d$ elements), which are eigenfunctions of $T_f$. We will see that, in the case of the eigenvalues of the matrix $T_n$, a finite sample analog holds. 

Recalling that a $L^2$ basis of eigenfunctions of $T_W$ is given by $\cup_{n\geq 0}\cup_{p_{k,n}\in Q_n}p_{k,n}$, the following estimate will useful for proving the concentration of the eigenvectors of the adjacency matrxi of the observed graph with respect to the eigenfunctions of $T_W$
\begin{lemma}\label{lem:inf_norm_estimates}
We have for $\{p_{k,n}\}^{ \dim(\mathcal{Y}_n)}_{k=1}\in\mathcal{Y}_n$\begin{align*}
\|p_{k,n}\|_\infty&\lesssim n^{\nu+\frac{d-1}{2}}\quad\text{ for }1\leq k\leq \dim(\mathcal{Y}_n)\\
\Big\|\sum^{\dim(\mathcal{Y}_n)}_{k=1}p^2_{k,n}\Big\|_{\infty}&\lesssim n^{2\nu+d-1}
\end{align*}
\end{lemma}

\begin{remark}
Notice that when $\nu=0$, the space of orthogonal polynomials $\Y_n$ coincide with $\H^d_n$, the space of classic spherical harmonics of degree $n$ in $\S^{d-1}$. Replacing $\nu=0$ in Lemma \ref{lem:inf_norm_estimates}, we recover the classic estimate $\|p^2_{k,n}\|_\infty\lesssim n^{d-1}=\dim(\H^d_n)$.
\end{remark}

We will use weighted Sobolev spaces to define regularity, which is similar to spherical context treated in \cite{Yohann},\cite{Ara2}.
\begin{defn}\label{defn:sobolev}
We will say that $f\in S^p_{\gamma_{\nu}}([-1,1])$ or, equivalently, that $f$ is Sobolev regular with parameter $p$, if and only if the eigenvalues $\lambda^\ast_n$ (given by \eqref{eq:eigs}) satisfy $\sum_{n\geq 0} |\lambda^\ast_n|^2 d_n(1+\nu_n^{p})<\infty$, where $\nu_n=n(n+2\nu+d-1)$.
\end{defn}

\subsection{Eigenvalues and eigenvectors}

We consider the integral operator $T_W:L^2(\B^d,F_\nu)\rightarrow L^2(\B^d,F_\nu)$ \[T_W g(x)=\int_{\B^d}f(\langle x,y\rangle) g(y)dF_\nu(y)\]
and the $n\times n$ symmetric matrices \begin{align*}
T_n&=\frac 1n(1-\delta_{ij}) f(\langle X_i,X_j\rangle)\\
\hat{T}_n&=\frac 1n A_{ij}
\end{align*}
where $A_{ij}$ is the adjacency matrix defined in Section \ref{sec:RGG_ball}. The following two results are key steps to prove that the spectrum of $\hat T_n$ is close to the spectrum of $T_W$ with high probability
\begin{itemize}
\item We will use Bandeira-Van Handel result \cite[Cor.3.12]{BanVan}, which provides a sharp concentration inequality for the operator norm of random matrices with independent entries. This allow us to say that  $\lambda(\hat{T}_n)$ is close to $\lambda(T_n)$ for $n$ large. Framed in our context, this results gives (see Thm. \ref{thm:bandeira_vanhandel})  \[\|\hat{T}_n-T_n\|_{op}\lesssim_\alpha \frac{1}{\sqrt n}\] with probability larger than $1-\alpha$. 
\item In order to prove that $\lambda(T_n)$ is close to $\lambda(T_W)$ in the finite sample sense, we will improve upon the concentration result \cite[Prop.4]{Yohann}, which says that, with high probability, $\delta_2(\lambda(T_n),\lambda(T_W))$ decrease as $n$ grows, with a nonparametric rate depending on Sobolev-type regularity conditions (analogous to Def.\ref{defn:sobolev}). More formally, in Proposition \ref{prop:delta_2} we prove that
\[\delta_2(\lambda(T_n),\lambda(T_W))=\mathcal{O}_\alpha(\frac{1}{\sqrt n})\] provided that $W$ is Sobolev regular with parameter $p>2\nu-1+\frac{5d}2$. 
\end{itemize}

\subsection{Eigengap condition} \label{sec:eigap}
In context of Gram matrix estimation, will assume the following eigengap condition. Given a geometric graphon $W$, we define the \emph{spectral gap} of $W$ relative to the eigenvalue $\lambda^\ast_1$ by \[\Delta^*(W):=\min_{j\neq 1}{|\lambda^\ast_1-\lambda^\ast_j|}\] which quantifies the distance between the eigenvalue $\lambda^\ast_1$ and the rest of the spectrum. We will drop the dependency on $W$ when is clear from the context. We will ask for $\Delta^*$ to be strictly positive, which will allow us to identify a cluster of eigenvalues of $\hat{T}_n$, that is a set of eigenvalues which are close to $\lambda^*_1$ and which are sufficiently isolated from the rest of the spectrum. In this case, the size of the cluster associated with $\lambda^*_1$ is exactly $\dim{(\Y_1)}$ (which is equal to $d$). 

We will define the following event \[\mathcal E:=\Big\{\delta_2\Big(\lambda\big(T_n\big),\lambda(T_W)\Big)
\vee \frac{2^{\frac92}\sqrt d}{\Delta^\ast}\|T_n-\hat{T_n}\|_{op}\leq\frac{ \Delta^\ast}4\Big\}\,,\]
for which we prove the following
\begin{lemma}\label{lem:event_2}
Assume that $\Delta^\ast>0$, then there exists $n_0\in\N$ such that for $n\geq n_0$ and for $\alpha\in(0,1)$ we have with probability larger than $1-\alpha$ \[\mathds{P}(\mathcal{E})\geq 1-\frac{\alpha}{2}\]
\end{lemma}
\begin{proposition}\label{prop:mainprop}
On the event $\mathcal E$, there exists one and only one set, consisting of $d$ eigenvalues of $\hat{T}_n$, whose diameter is smaller than $\Delta^\ast/2$ and whose distance to the rest of the spectrum of $\hat{T}_n$ is at least $\Delta^\ast/2$. 
\end{proposition}

We now give the main result of gram matrix estimation, which is a more precise version of Thm.\ref{thm:main3_informal}.
\begin{theorem}\label{thm:main3}
Let $W$ be a graphon defined by a dot product kernel on $\B^d$ and measure $F_\nu\in\F$. If $W\in S^p_{\gamma_{\nu}}([-1,1])$ for $p>2\nu-1+\frac{5d}2$ and we assume that $\Delta^*(W)>0$, then there exists a set of $d$ eigenvalues $\hat{v}_1,\cdots,\hat{v}_d$ of $\hat{T}_n$, such that with probability larger than $1-\alpha$ we have
\begin{equation}\label{eq:main_thm}\|\mathcal{G}^*-\hat{\mathcal{G}}\|_F\lesssim_\alpha {\Delta^*(W)}^{-1}\frac{1}{\sqrt n}\end{equation}
where $\hat{\mathcal{G}}=\mathcal{G}_{\hat{V}}$ and $\hat{V}$ is the matrix with columns $\hat{v}_1,\hat{v}_2,\cdots, \hat{v}_d$. 
\end{theorem}

\begin{remark}
The condition $p>2\nu-1+\frac{5d}2$ is mainly technical, as it is a sufficient condition for having a parametric rate for $\delta(\lambda(T_n),\lambda(T_W))$, see Proposition \ref{prop:delta_2}. The dependency on alpha in \eqref{eq:main_thm} is through a multiplicative constant of the form $\log{1/\alpha}$. 
\end{remark}

\section{The sough after power law distribution}\label{sec:power law}

As we already announced in the introduction, the RGG model on the ball is more flexible that the one on the sphere, in terms of observed degree distribution, when we restrict ourselves to spherically symmetric distributions (which is desirable in light of the previous chapter). Given that the heterogeneity in the degree sequence is a characteristic observed in many real world networks \cite{Barrat}, having a more flexible model at hand would be useful for modeling purposes. From \eqref{eq:meandeg} we see that the possible values for $d_W(x)$ for a threshold graphon $W$ are of the form $F_\nu(\operatorname{Sc}(N,1-\frac{\tau}{\|x\|\vee\tau}))$. However useful the previous characterization might be, it has the problem of not being very explicit and as it is written do not match any of the typical degree distributions that are frequent in the network literature. In particular, many real networks exhibit a power law degree distribution \cite{Newman,Mitzen}, meaning that the number of nodes with (unnormalized) degree $k$ is proportional to $k^{-\gamma}$ with $\gamma>0$. This opens the question: is the power between the possible degree distributions in the RGG model on the ball? or at least, is it possible to have an approximative version of it?

We will consider the following RGG on $\B^d$, defined by the connection function:  \[f(t)=\begin{cases}\frac{\alpha}{|t|^2}\wedge 1& \text{ if }t\neq 0\\ 
1 & \text{ if } t=0
\end{cases}\] where $\alpha\in (0,1)$ is a ``resolution'' parameter. For the latter we mean that if $x\in \B^d(0,\sqrt{\alpha})$ then for all $y \in \B^d$ we have $f(\langle x,y\rangle)=1$. That is, any point located in the ball centered in $0$ with radius $\alpha$ will connect with every other point in $\B^d$. This is the inverse of what happens in the threshold graphon. In terms of the degree function, this means that $d_f(x)=1$ for $\|x\|^2\leq\alpha$. By defintion we have 
\begin{align*}
    d_f(x)=\int_{\B^d}\frac{\alpha}{|\langle x,y\rangle|^2}\wedge 1 dF_\nu(y)
\end{align*}
By rotational symmetry (we can think of $x$ being $x=(x_1,0,\cdots,0)=x_1 N$) we have 

\begin{align*}
d_f(x)&=d_f(x_1N)\\
&=\int_{\B^d}\frac{\alpha}{x^2_1y^2_1}\wedge 1dF_\nu(y)\\
&=\int_{\B^d\setminus \B^d(0,\sqrt \alpha) }\frac{\alpha}{x^2_1y^2_1}\wedge 1dF_\nu(y)+\int_{\B^d(0,\sqrt \alpha)}dF_\nu(y)
\end{align*}

Given that the summand $\int_{\B^d(0,\sqrt \alpha)}dF_\nu(y)$ is common to every $x\in\B^d$ we will subtract it. Intuitively speaking, there will be nodes that will be connected will almost every node in the graph, which increase the minimum degree. Since we already know that the nodes such that $\|x\|^2\leq \alpha$ are connected with every other node, we concentrate in the case $\|x\|^2>\alpha$. 
This motivates the definition \[\tilde{d}_f(x):=\frac{\int_{\B^d\setminus \B^d(0,\sqrt \alpha) }\frac{\alpha}{x^2_1y^2_1}dF_\nu(y)}{\int_{\B^d\setminus \B^d(0,\sqrt \alpha)}\frac{1}{y_1^2}dF_\nu(y)}\] for $\|x\|^2>\alpha$. Let take $k,n'\in \N$ such that $k<n'\leq n$. By definition \begin{align*}
\tilde{d}_f(\sqrt{\frac{n'\alpha}{k}}N)&=\frac{\int_{\B^d\setminus \B^d(0,\sqrt \alpha) }\frac{k}{n'y^2_1}dF_\nu(y)}{\int_{\B^d\setminus \B^d(0,\sqrt \alpha)}\frac{1}{y_1^2}dF_\nu(y)}\\
&=\frac{k}{n'}
\end{align*}

Thus, given that $X$ is distributed according to $F_\nu$ \begin{align*}
  \P\big(\frac{k-1}{n'} \leq d_f(X)\leq \frac k{n'} \big)&=\P\big(X\in\B^d(0,\sqrt{\frac{n'\alpha}{k-1}})\setminus\B^d(0,\sqrt{\frac{n'\alpha}k})\big) \\
  &=F_\nu\Big(\B^d(0,\sqrt{\frac{1}{k-1}})\setminus\B^d(0,\sqrt{\frac1k})\Big)\\
  &=I_{\frac{n'\alpha}{k-1}}(\frac d2,\nu+\frac12)-I_{\frac{n'\alpha}{k}}(\frac d2,\nu+\frac12)
\end{align*}
Since $I_x(\cdot,\cdot)$ is an increasing function we see that \[\Delta_k\frac{d}{dx}I(\frac{n'\alpha}{k-1})\leq I(\frac{n'\alpha}{k-1})-I(\frac{n'\alpha}{k})\leq \Delta_k\frac{d}{dx}I(\frac{n'\alpha}{k}) \]
where $I(x)=I_x(\frac d2,\nu+\frac12)$ and $\Delta_k=\frac{n'\alpha}{k-1}-\frac{n'\alpha}{ k}$. It follows from the definition that $\frac{d}{dx}I_x(a,b)=\frac{1}{B(a,b)}x^{a-1}(1-x)^{b-1}$ and consequently \[c_{\nu,\alpha}\big(\frac{k-1}{n'}\big)^{-\frac d2-1}\leq I(\frac{1}{k-1})-I(\frac{1}{k})\leq C_{\nu,\alpha}\big(\frac{k}{n'}\big)^{-\frac d2-1} \]
where $c_{\nu,\alpha}$,$C_{\nu,\alpha}$ are constants that depend on $\nu$ and $\alpha$. Thus picking $d=3$, for example, we have that \[\P(\frac{k-1}{n'}\leq \tilde{d}_f(X)\leq \frac{k}{n'})\propto (k/n')^{-2.5}\] which can see as a similar distribution to a power law\footnote{ A random graph model has power law if the number of nodes with (unnormalized) degree $k$ is proportional to $k^{-\gamma}$ with $\gamma>0$.}. To make this point clearer, take for example $n'$ of order $\mathcal{O}(n)$.
 The previous can be interpreted as the proportion of nodes of degree $k$, for $k$ large, follows a power law function, after shifting. The exponent $-2.5$ has been frequently reported in the literature for real networks \cite{Newman}. We see that that changing the exponent $\alpha$ in the definition of $f(t)$ and the changing the dimension of the sphere, we can fine-tune the power law exponent parameter. 

\begin{figure}[h!]
\centering
\begin{tikzpicture}
\draw (0,0) circle (2.5cm);
\coordinate (O) at (0,0);
\draw[dashed] (0,0) circle (0.4cm);
\draw(0,0) --(-0.4cm,0);
\draw[decorate,decoration={brace,amplitude=2pt}] (0,0) --(-0.4cm,0)node[black,midway,xshift=0.01cm,yshift=-0.2cm]{{\tiny $\sqrt{\alpha}$}};
\coordinate (M) at (40:2.5);
\coordinate (N) at (-40:2.5);
\draw[dashed] (0,0) circle (0.8cm);
\draw[dashed] (0,0) circle (0.5cm);
\fill[blue!50!cyan,opacity=0.3,even odd rule](0,0) circle (0.8cm) (0,0) circle (0.5cm);
\coordinate (B1) at (0,0.6);
\node at (B1) [right]{$\scriptsize{x}$};
\fill(B1) circle[radius=1.2pt] node[below left] {};
\coordinate (B2) at (0.8,0);
\node at (B2) [right] {$\scriptsize \tilde{d}_f(x)\sim k/n'$};

\end{tikzpicture}
\caption[Illustration of the connection between geometry and the degree function.]{A point $x$ inside the annulus between the circles $\sqrt{\frac{n'\alpha}{k-1}}$ and $\sqrt{\frac{n'\alpha}{k}}$ will have degree function satisfying $\tilde{d}_f(x)\sim k/n'$. The fraction of points with degree $k/n'$ is the measure of the annulus.}\label{fig:M2}
\end{figure}
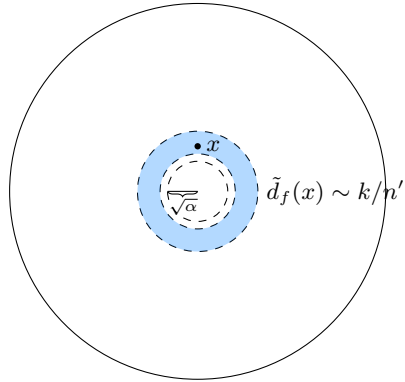

The following proposition proves that $f$ has power law tails also in the sense of Definition \ref{defn:power-law}

\begin{proposition}\label{prop:power-law_equiv}
There exist a constant $C>0$ such that the degree function $d_f(x)$ satisfies \[F_\nu(\{x\in\B^d: d_f(x)\geq h\})\leq C(h-\kappa)^{-\theta}\]
for $\theta=1.5$, $\kappa=\int_{\B^d(0,\sqrt \alpha)}dF_\nu(y)$.
\end{proposition}

The convergence of the cumulative distribution of the degrees is proven in \cite{Delmas} and \cite{Chayes2}. In \cite{Chayes2}, the authors prove the convergence of $|\{i\in[n]: \frac{d_G(X_i)}{n}>h\}|$ towards $\mu(\{y:d_W(y)>h\})$, where $\mu$ is the distributions of the points $X_i$ and $\lambda>0$ is a point of continuity of $\lambda~\rightarrow~ \mu(\{d_W(y)~>~h\})$. In \cite{Delmas}, a graphon $W$ in $[0,1]$ is considered and the convergence of $\Pi(y)~:=~\frac 1n \sum^n_{i=1}\mathbbm{1}_{d_G(X_i)<nd_W(y)}$ toward $y$, almost surely, is proved. They also provide a CLT type result for this convergence. 

From the previous we can deduce that if $\{G_n\}_{n\geq 1}$ is a sequence of graphs, obtained by sampling the graphon $W(x,y)=f(\langle x,y\rangle)$ by the $W$-random graph model, as decribed in Sec. \ref{sec:RGG_ball}, then there exists a $n_0\in\N$ such that all the graphs $\{G_n\}_{n\geq n_0}$ will have a (discrete) power law type distribution. Indeed, this is consequence
of the characterization of $F_\nu~(\{x\in\B^d~:~d_W(x)~\geq~ h \})$ given in this section, the fact that the sequence $G_n$ will converge to $W$ in the cut distance \cite{Lov} and \cite[Prop.21]{Chayes2}.

\section{Numerical Experiments}\label{sec:num_exp4}

We run simulation for different RGG models on $\B^d$ and compute the estimators analyzed throughout this paper (for the norm and the Gram matrix) to see how they perform empirically. In the case of the Gram matrix estimation, we run the algorithm HEiC described in \cite[Sec.3]{Ara}, changing the normalization constants as described in Section \ref{sec:eigensystem}. 

We start by considering simulations for the link function that gives a power-law type degree distribution, considered in Section \ref{sec:power law}. 

\subsection{Power law type graphon}

We consider the link function  \[f(t)=\begin{cases}\frac{\alpha}{|t|^2}\wedge 1& \text{ if }t\neq 0\\ 
1 & \text{ if } t=0,
\end{cases}\]  to illustrate empirically the behavior of the degree profile under this model. In Figure \ref{fig:degreeprofile4} we show a single simulation of the graph of size $3000$ with connection function $f(\cdot)$ and parameter $\alpha=1/1000$, under the measure $F_{1/2}(\cdot)$, which we recall is the uniform measure in $\B^d$. We observe the presence of nodes with very high degree (or large hubs) which is often observed in real world networks and scale-free networks. We include a $\log-\log$ plot for the histogram for the nodes with degrees over $300$, to better observe the exponential decay. The resulting shape, first close to a line and then oscillatory (Figure \ref{fig:degreeprofile4}(right)), has been reported in real world networks, where it  is suggested as evidence for a power law distribution of the degrees \cite{Newman}. 
\begin{figure}[ht]
\centering
\begin{minipage}[c]{3.1cm}
\centering
\includegraphics[scale=0.21]{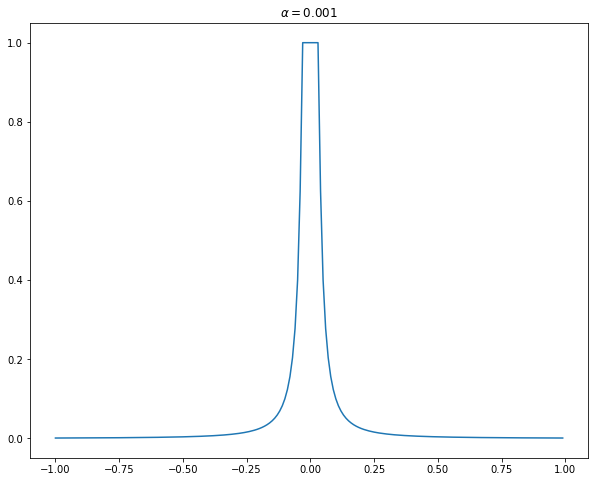}
\end{minipage}%
\hspace{1.6cm}
\begin{minipage}[c]{3.1cm}
\centering
\includegraphics[scale=0.21]{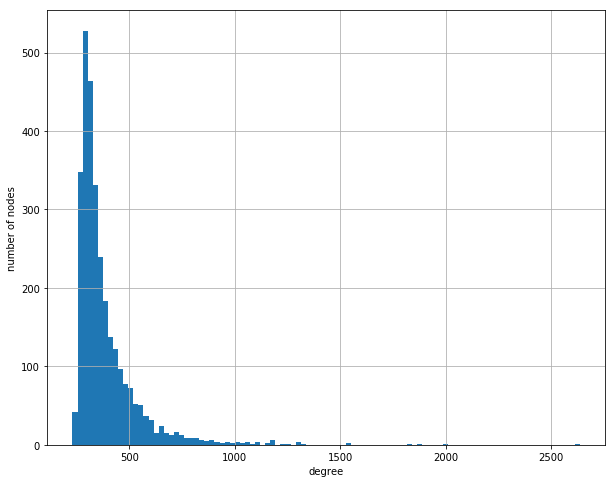}
\end{minipage}%
\hspace{2cm}
\begin{minipage}[c]{3.1cm}
\centering
\includegraphics[scale=0.21]{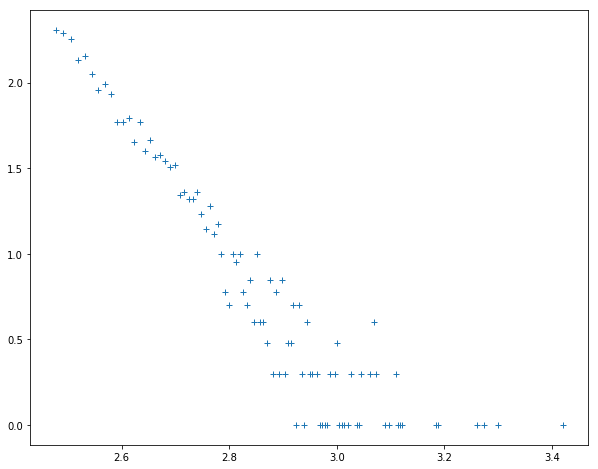}
\end{minipage}
\caption[Degree histogram for the connection function $f(t)=\frac{\alpha}{|t|^2}\wedge 1$ ]{ From left to right: we plot the function $f(\cdot)$ for $\alpha=1/1000$. In the center, we show the histogram for this $f$. In the right, we show a $\log-\log$ plot of the same histogram, but only for the values with degree larger than $300$.}
\label{fig:degreeprofile4}
\end{figure}

We repeat this exercise in Figure \ref{fig:degreeprofile5}, for different values of $\alpha$ which produces changes in the distribution. We opt to include the $\log-\log$ plot for nodes with degree larger than $300$ for comparison purposes. This shows the shifted power law shape of the degree distribution. More node connectivity can also be achieved by changing the measure under which we simulate. We show one example in the image at the bottom in Figure \ref{fig:degreeprofile5}, which was generated with $F_{3/2}$. Indeed, a measure that allocates more mass at the center, will have larger connectivity within this model. This serves to illustrate the flexibility and expressiveness of this model.

\begin{figure}[ht!]
\centering
\begin{minipage}[c]{3.1cm}
\centering
\includegraphics[scale=0.21]{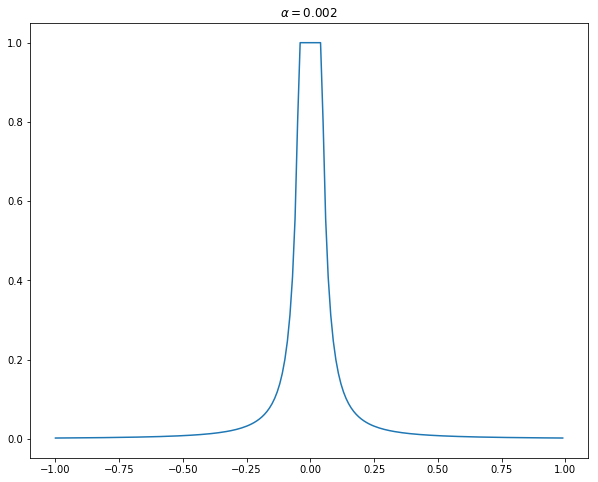}
\end{minipage}%
\hspace{1.6cm}
\begin{minipage}[c]{3.1cm}
\centering
\includegraphics[scale=0.21]{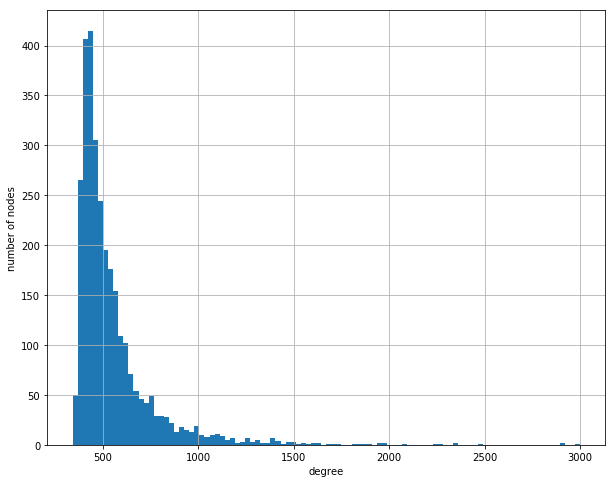}
\end{minipage}%
\hspace{2cm}
\begin{minipage}[c]{3.1cm}
\centering
\includegraphics[scale=0.21]{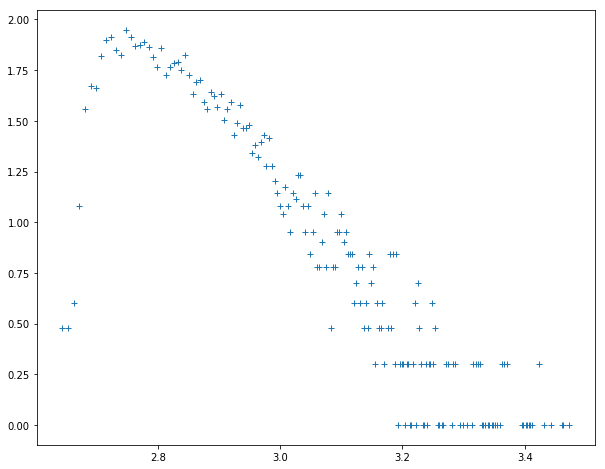}
\end{minipage}
\begin{minipage}[c]{3.1cm}
\includegraphics[scale=0.21]{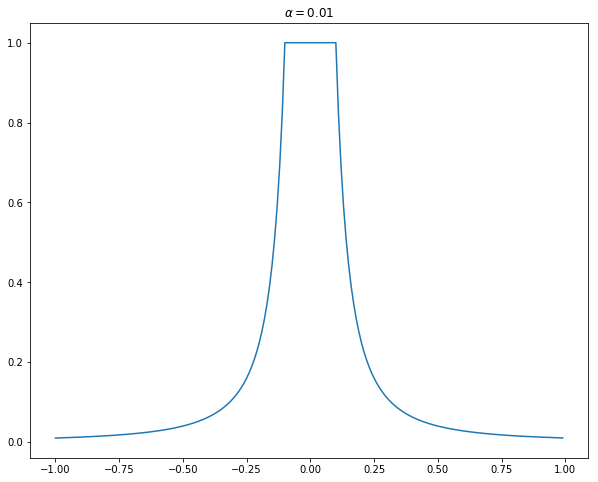}
\end{minipage}%
\hspace{1.6cm}
\begin{minipage}[c]{3.1cm}
\centering
\includegraphics[scale=0.21]{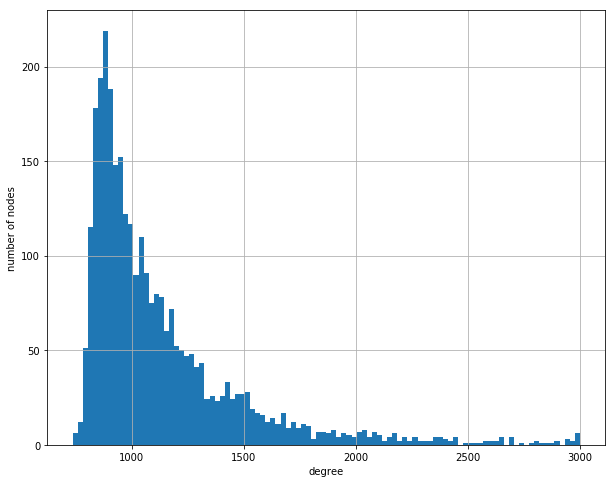}
\end{minipage}%
\hspace{2cm}
\begin{minipage}[c]{3.1cm}
\centering
\includegraphics[scale=0.21]{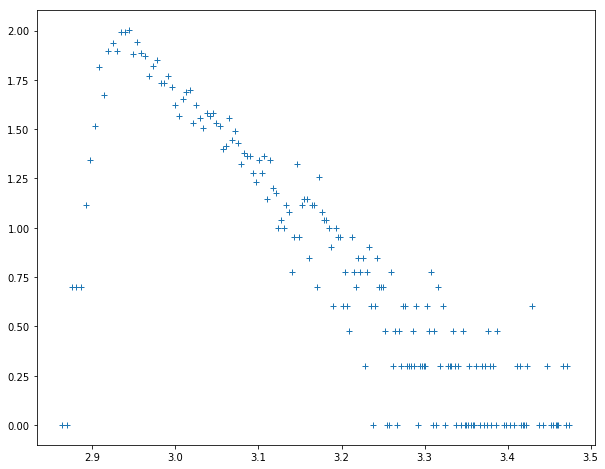}
\end{minipage}
\begin{minipage}[c]{3.1cm}
\includegraphics[scale=0.21]{gr_ch3_alpha1.png}
\end{minipage}%
\hspace{1.6cm}
\begin{minipage}[c]{3.1cm}
\centering
\includegraphics[scale=0.21]{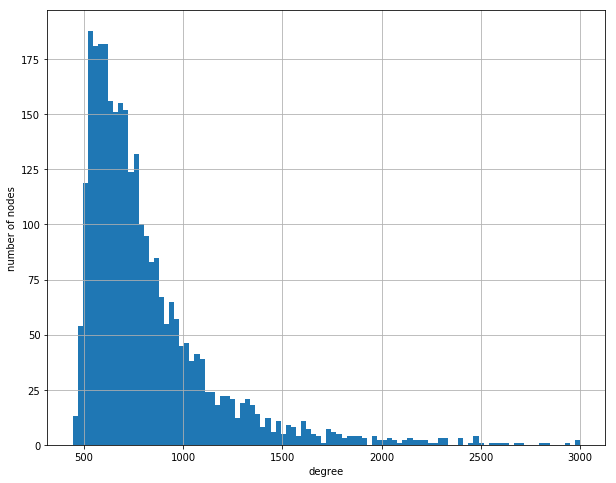}
\end{minipage}%
\hspace{2cm}
\begin{minipage}[c]{3.1cm}
\centering
\includegraphics[scale=0.21]{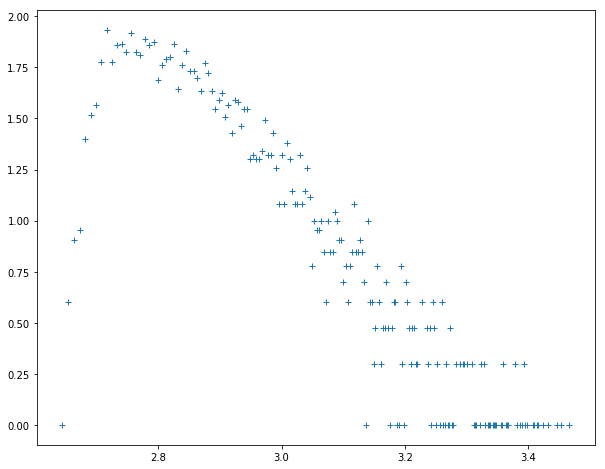}
\end{minipage}
\caption[Another degree histogram.]{ Analogous to Figure \ref{fig:degreeprofile4}, using different values of $\alpha$. In the $\log-\log$ plot we show the values with degree larger than $300$. The plot at the bottom was done with a different measure $F_{3/2}$}
\label{fig:degreeprofile5}
\end{figure}

\subsection{Latent norm recovery}

We study the empirical performance of the estimator $\tilde{\zeta}_i$, for each $1\leq i\leq n$, for which we proved almost sure converge to the latent norm on the threshold RGG model. We compute the estimator for each node in the graph, following measure of error for each sample \[E_{norm}=\frac{1}{\sum^n_{i=1}\mathbbm{1}_{\|X_i\|\geq \tau}}\Big(\sum^n_{i=1}(\zeta_i-\|X_i\|)^2\mathbbm{1}_{\|X_i\|\geq \tau}\Big)^{1/2}\]

We discard the points with norm larger than $\tau$, because as discussed in Section \ref{sec:main_results} the adjacency matrix of the graph carries no information about the norm of those points, other than being smaller than $\tau$. In Figure \ref{fig:errornorm1}(left) we plot the mean square error $E_{norm}$ in logarithmic scale for a threshold $\tau=0.1$. For each sample size, we simulate $25$ graphs on the ball with dimension $d=3$, and uniform measure $F_{1/2}$, and compute the mean of the errors $E_{norm}$. The form in which the error decrease suggest a parametric rate of convergence, which we plot in a red line for reference. However, note that the fact the estimator is based in a complicated nonlinear function, as it is \[t\rightarrow \frac{\tau}{\sqrt{1-I^{-1}(t;\frac d2+\nu,\frac 12)}}\]
makes that this rate is non-uniform across the nodes. Indeed, given the shape of the graph of $I^{-1}(t;\frac d2+\nu,\frac 12)$ it is not hard to see that points in with higher norm (closer to $1$) will converge slower. This indeed what we observe in the experiments as shown in Figure \ref{fig:errornorm1}(right), where we plot the sequence of ordered norms in red and the sequence of ordered $\zeta_i's$ for different values of the sample size ($n=100,$). Notice that it takes much more samples to see a convergence when the norm of the node is close to $1$.  

We observe that for values of $\tau$ closer to $0$, the convergence is indeed slower. In Figure \ref{fig:errornorm1} we plot the mean of $\log(E_{norm})$ over $25$ sampled graphs, for a threshold $\tau=0.01$ with dimension parameter $d=3$ and the measure $F_{1/2}$. We observe that it takes much more samples to converge. Even if the decrease of the errors suggest a similar parametric rate in the case of the model with smaller $\tau$, the constant (intercept) is larger, which means that the error is always larger than in the previous case. This should not be surprising given that we know that in the model with $\tau=0$ we cannot infer the norm from the samples (as the model is equivalent to the threshold graphon on the sphere). Approaching to $\tau=0$ will render the problem harder, in the sample complexity sense.  

\begin{figure*}[ht!]
\centering
\begin{minipage}[c]{0.5\textwidth}
\centering
\includegraphics[width=\linewidth,height=133pt]{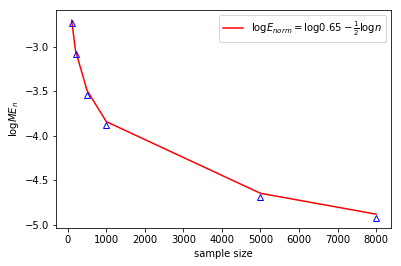}
\end{minipage}%
\begin{minipage}[t]{0.5\textwidth}
\centering
\includegraphics[width=0.5\linewidth]{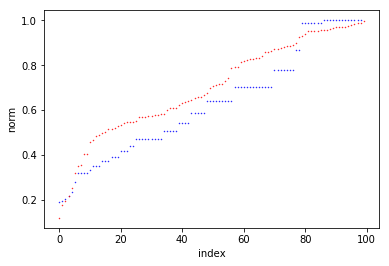}%
\centering
\includegraphics[width=0.5\linewidth]{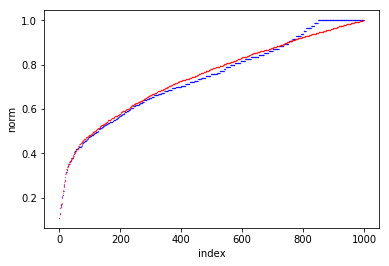}
\centering
\includegraphics[width=0.5\linewidth]{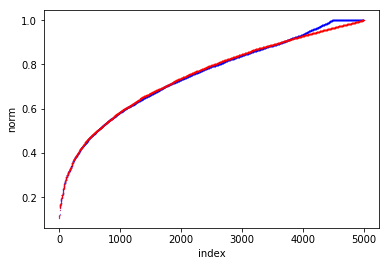}%
\centering
\includegraphics[width=0.5\linewidth]{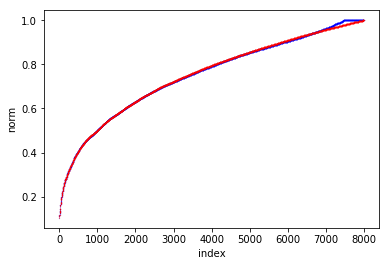}
\end{minipage}%

\caption[Mean error for the norm recovery.]{ In the left we show the mean of $\log(E_{norm})$ over $25$ graphs, for the recovery of the norm on the threshold graph with $\tau=0.1$ and parameter $d=3$ and measure $F_{1/2}$. We add the upper in red for reference. In the right we plot the sequence of ordered values for $\|X_i\|$ in increasing order together with the sorted sequence of estimated values $\zeta_i$.  }
\label{fig:errornorm1}
\end{figure*}

\begin{figure*}[ht!]
\centering
\begin{minipage}[c]{0.5\textwidth}
\centering
\includegraphics[width=\linewidth,height=133pt]{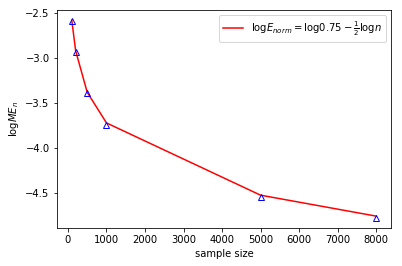}
\end{minipage}%
\begin{minipage}[t]{0.5\textwidth}
\centering
\includegraphics[width=0.5\linewidth]{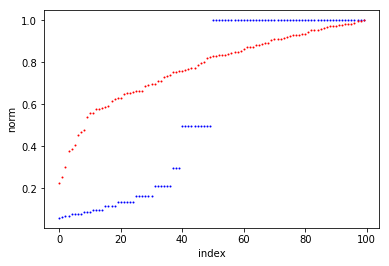}%
\centering
\includegraphics[width=0.5\linewidth]{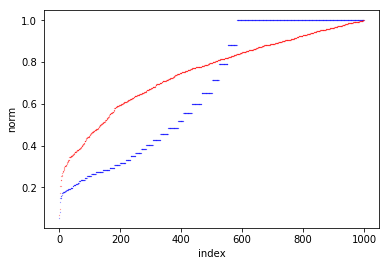}
\centering
\includegraphics[width=0.5\linewidth]{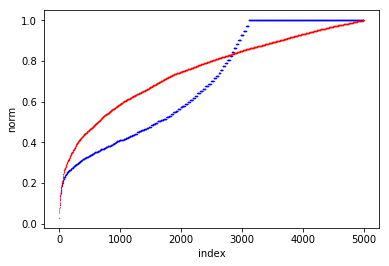}%
\centering
\includegraphics[width=0.5\linewidth]{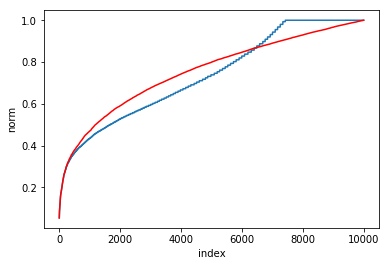}
\end{minipage}%

\caption[Another example of norm recovery.]{ This is analogous to Figure \ref{fig:errornorm1}, with $\tau=0.01$ and maintaining the rest of parameters.}
\label{fig:errornorm2}
\end{figure*}

To see empirically the effect of changing the parameter $\nu$ in the estimation of the norm, in Figure \ref{fig:nus} we plot the mean of the error $\log(E_{norm})$ across $25$ samples for the threshold graphon with $\tau=0.1$ and $d=3$. We see that a larger $\nu$ gives lower error, this is explained by the fact the larger the $\nu$, the more concentrated the sampled nodes are close to the center of the ball. We added, for reference, the plot of the theoretical density of the (squared) norm of the latent points (a Beta distribution by Lemma \ref{lem:dist_norm})  in Figure \ref{fig:nus} (right).

\begin{figure}[h!]
\centering
\begin{minipage}[c]{0.5\textwidth}
\includegraphics[scale=0.45]{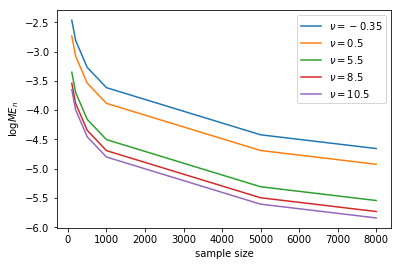}
\end{minipage}%
\begin{minipage}[c]{0.5\textwidth}
\includegraphics[scale=0.45]{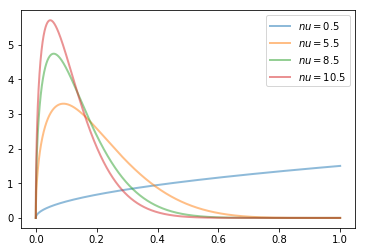}
\end{minipage}
\caption[Similar to Figure \ref{fig:errornorm2}]{Similar to Figure \ref{fig:errornorm2}. We plot the error for the threshold graphon with $\tau=0.1$, $d=3$ and changing the measure $F_\nu$. In the right we plot the pdf of $Beta(\frac{d}{2},\nu+\frac{1}{2})$ which is the distribution of the norm of the nodes.}
\label{fig:nus}
\end{figure}

\subsection{Gram matrix estimation}

We report the empirical performance of the algorithm HEiC, described in \cite{Ara} applied in the context of RGG in $\B^d$. Similar to the spherical case, we will mainly measure the mean error \[ME_n=\|\mathcal{G}^*-\mathcal{G}\|_F\]

We first consider the threshold graphon with parameter $\tau=0.1$ in dimension $d=3$. We sample $25$ graphs using this model and run each time the algorithm HEiC. In Figure \ref{fig:error3_latdist_1}(left) we show a boxplot for $\log(ME_n)$ for different sample sizes. In Figure \ref{fig:error3_latdist_1}(right) we show the $\log(ME_n)$ error for different values of $n$ in the case of the logistic graphon $f(t)=\frac{1}{1+e^{rt}}$ for different values of $r$. The curves in the plot were obtaining by averaging across $25$ samples for each value of $n$. We observe that for $r=-0.1$ the error does not decrease with the sample size, which is to be expected as the logistic function for that value of $r$ is close to a constant function. In our parametrization of the problem, this translate into a close to $0$ spectral gap as the Figure \ref{fig:gap} illustrates. Indeed, we plot the first $10$ eigenvalues, for this case the cluster of eigenvalues associated to $\lambda^*_1$ is a subset of the first $10$ eigenvalues. We see that as $r$ is closer to zero, the spectral gap decrease and the number of samples required to have a decreasing error increase. 

\begin{figure*}[ht!]
\centering
\begin{minipage}[t]{0.5\textwidth}
\centering
\includegraphics[width=\linewidth,height=133pt]{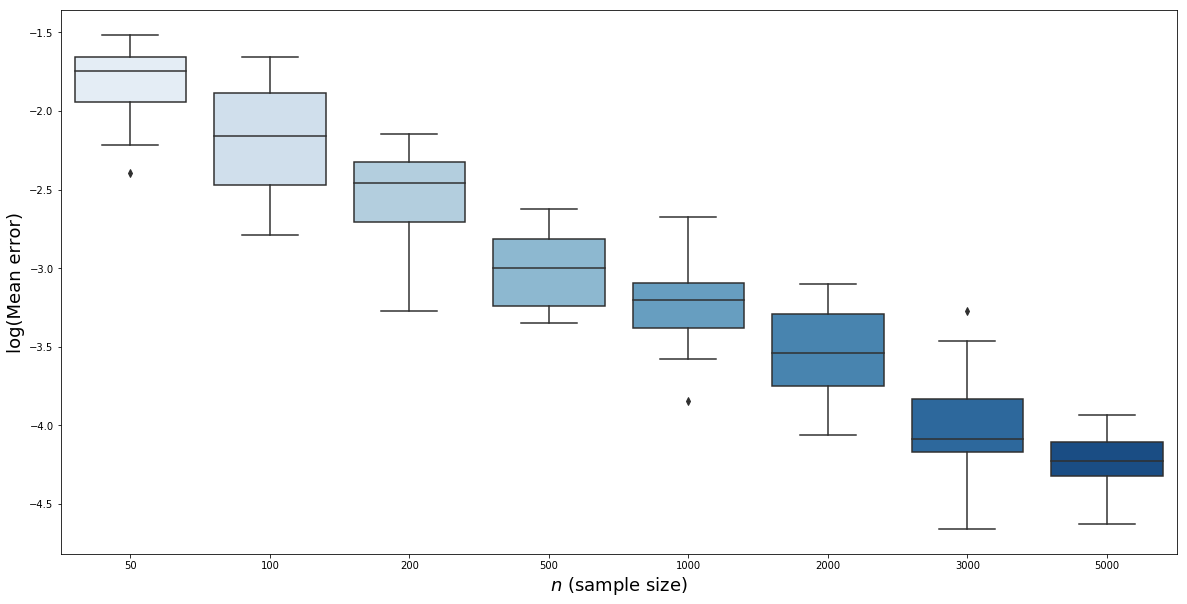}
\end{minipage}%
\begin{minipage}[t]{0.5\textwidth}
\centering
\includegraphics[width=\linewidth,height=138pt]{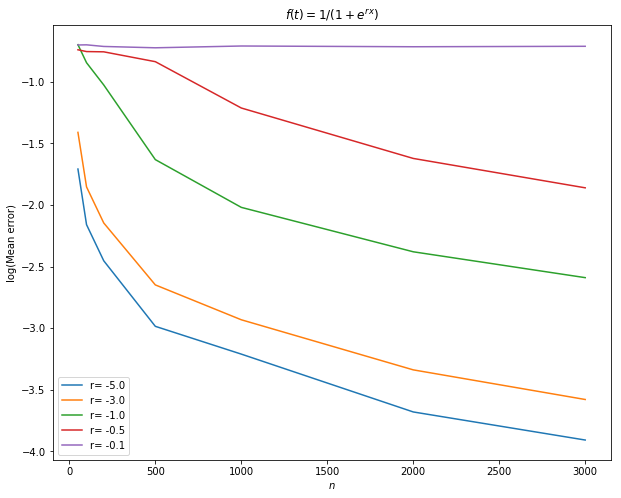}%
\end{minipage}
\caption[Mean error for the latent distances estimation on $\B^d$.]{ In the left a boxplot for $\log(ME_n)$, for different values of the sample size in threhsold graphon with $d=3$, $\tau=0.1$ and $F_{1/2}$. In the right we plot the mean error for the logistic graphon with different values of the parameter $r$.}
\label{fig:error3_latdist_1}
\end{figure*}

\begin{figure}
\centering
\includegraphics[scale=0.3]{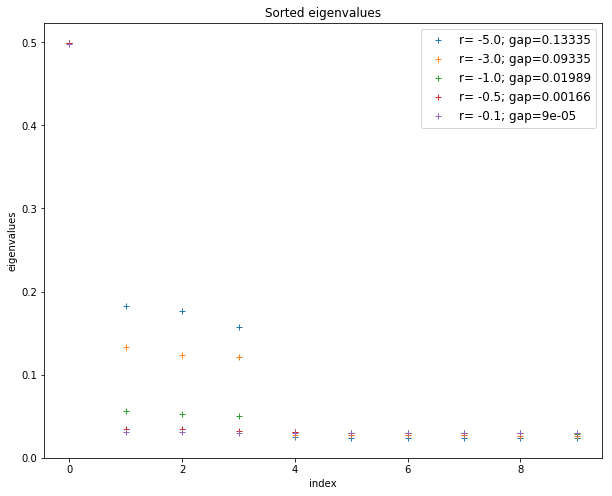}
\caption[Eigenvalues and eigengap of the logistic graphon.]{We plot the first $10$ eigenvalues for the logistic graphon for different values of the parameter $r$. We include the spectral gap in each case. In all the examples we used a parameter $n=1000$. }
\label{fig:gap}
\end{figure}

Note that Theorem \ref{thm:main3} do not give information about the diagonal of the Gram matrix, which corresponds to the square of the norms of the nodes $X_i$. Our measure of error $ME_n$ do not take them into consideration. In the case of the threshold graphon we can use the estimator $\zeta_i$ to compute the means. We observed empirically that the algorithm works better when the rows matrix of eigenvectors $\hat{V}$, which has columns $v_1,\cdots ,v_d$ which are the output of the algorithm HEiC, are normalized to match the mean of the true means $\zeta_i$. This is not an ideal situation from the practical point of view, given that the norms are usually non available. In the case of the threshold graphon we can use the estimated norms in this extra normalization step. While this gives reasonable results in practice, a thorough theoretical study is lacking at this moment and will be left for future work. 

\begin{remark}
The time complexity(or running time) of the latent distance recovery algorithm does not increase, in comparison with the spherical case, and it is roughly $\mathcal{O}(n^3+n)$. In the case of the computation of the estimators $\zeta_i$ we need to compute the degrees, which corresponds to compute the sum of all rows, which is roughly $\mathcal{O}(n^2)$.
\end{remark}

\section{Conclusions and future work}

 We studied the problem of estimating the norm and the Gram matrix for the latent points of graphs generated by the RGG model on $\B^d$. Extending the approach of Proposition \ref{prop:conv_est_norm} to known (given) link functions other than the threshold function is possible, because in that case we will have an expression analog to \eqref{eq:meandeg}. On the other hand, we expect that the use of global information, in conjunction with the degree function, would help us to find simpler estimators which are more prone to be studied, in the finite sample setting, with the standard concentration tools. 

The problem of estimating $\tau$, under the model with threshold link function $W_g(x,y)~=~\mathbbm{1}_{\langle x,y\rangle\geq \tau}$ (for a given $F_\nu$), is also of interest. This problem has been studied in the model with $\Omega=[-1,1]^2$ and link function $\mathbbm{1}_{\|x-y\|\leq r(n)}$ in \cite{Mcdiarmid}, where the uniform distribution is considered, but the model allows for sparser graphs. They propose an estimator based on the explicit formula for expectation of the number of edges. In the context of the model we presented here, we believe that simpler estimators are possible, given the fact isolated nodes give information about $F_\nu(\B^d(0,\tau))$. The main difficulty will be to estimate, with high probability, the number of isolated nodes whose associated points are outside $\B^d(0,\tau)$. 
Constructing such estimators in left for future work. 

Another interesting problem will be the estimation of the parameter $\alpha$ for the link function \[f(t)=\begin{cases}\frac{\alpha}{|t|^2}\wedge 1& \text{ if }t\neq 0\\ 
1 & \text{ if } t=0,
\end{cases}\] which present a power-law type distribution for the degree. Finding a larger class of link functions satisfying this property, and a proper description of this class, will also be of interests. Eventually, the problem could be framed as a non-parametric graphon estimation and, given the Fourier-Gegenbauer decomposition analyzed in Sec.\ref{sec:eigensystem}, it will be possible to use an spectral approach similar to the one developed in \cite{Yohann}.   

\bibliographystyle{apalike}
\bibliography{These.bib}

\newpage

\appendix

\section{Useful results}
Here we gather some of the results used through out paper.
\begin{lemma}\label{lem:dist_norm}
If $X$ is a $\B^d$-valued random variable distributed according to $F_\nu$, then $\|X\|^2$ follows a distribution $Beta(\frac d2,\nu+\frac12)$.
\end{lemma}

\begin{lemma}[Threshold graphon degree density]
Let $W$ be the threshold graphon and $f$ the probability density function of $d_{W}(X)$, where $X\sim F_\nu$ we have for $t>0$
\begin{equation}\label{eq:densi_geo}
f(t)=\frac{\tau^{d}}{(1-I^{-1}(2t))^{\frac d2+1}}\big(1-\frac{\tau^2}{1-I^{-1}}\big)^{\nu-\frac12}\frac{1}{I(2t)^{\frac d2+\nu}(1-I(2t))^{-\frac12}}
\end{equation}
where we use the notation $I(t)=I_{t}(\frac d2+\nu,\frac 12)$. 
\end{lemma} 
\begin{proof}
It is well known that the function $t\rightarrow I_t(a,b)$ is differentiable and it is straightforward to check that $g(t)$ it is also differentiable  for $t>0$. Taking the derivative of $F_{d_W}(t)=I_{g(t)}(\frac d2,\nu+\frac12)$ the result follows from simple computations \begin{align*}
f_g(t)&=\frac1{B(\frac{d}2,\nu+\frac{1}{2})} g(t)^{\frac{d}2-1}(1-g(t))^{\nu-\frac12}g^\prime(t)\\
&=\frac{\tau^{d}}{(1-I^{-1}(2t))^{\frac d2+1}}\big(1-\frac{\tau^2}{1-I^{-1}}\big)^{\nu-\frac12}\frac1{I(2t)^{\frac d2+\nu}(1-I(2t))^{-\frac12}}
\end{align*}
\end{proof}
\begin{lemma}\label{lem:degree_fnct}
For $\tau\geq 0$ we have for $0\leq t\leq 1$ \[d_W(tN)=\frac12 I_{1-\big(\frac{\tau}{t\vee\tau}\big)^2}\big(\nu+\frac d2,\frac12\big)\]
where $I_{x}(a,b)=\frac{1}{Beta(a,b)}\int^x_0t^{a-1}(1-t)^{b-1}$ is the regularized incomplete Beta function.
\end{lemma}
\begin{lemma}\label{lem:continuous_deg}
Let $W(\langle x,y\rangle)=\mathbbm{1}_{\{\langle x,y\rangle\geq \tau\}}$ be a graphon on $\B^d$, with $0<\tau<1$, and $F_\nu\in\F$. Then the function \[h\rightarrow \F_\nu(\{x\in\B^d: d_W(x)\geq h\}) \] is continuous in $(0,h^*]$, where $h^*:=\F_\nu(\B^d\setminus \tau\B^d)$.
\end{lemma}
\begin{proof}
By Lemma \ref{lem:degree_fnct} we have \[d_W(tN)=\frac12 I_{1-\big(\frac{\tau}{t\vee\tau}\big)^2}\big(\nu+\frac d2,\frac12\big)\] from which we see that $t\rightarrow d_W(tN)$ is strictly increasing on $(\tau,1]$ and its range is $(0,h^*]$. Then for any $h\in(0,h^*]$ there exists $t_h$ such that $\{x\in\B^d: d_W(x)\geq h\}=\B^d\setminus t_h\B^d$. Moreover, $t_h=\frac{\tau}{\sqrt{1-I^{-1}(2h)}}$. It is easy to see that $h\rightarrow t_h$ is continuous on $(0,h^*]$ and given that $F_\nu$ is absolutely continuos with respect to the Lebesgue measure, we have that $h\rightarrow \F_\nu(\{x\in\B^d: d_W(x)\geq h\})$ is continuous in $(0,h^*]$.    
\end{proof}

The following result gives a characterization for a basis of $\mathcal{V}_l$. The proof can be found in \cite[Thm.11.1.12]{Dai}

\begin{theorem}\label{thm:basis_charac}
The space $\mathcal{V}_n$ has a basis consisting on functions $G^{\gamma_\nu}_l(\langle x,\psi_i\rangle)$ for some points $\{\psi_i\}_{1\leq i \leq \dim(\mathcal{V}_l)}\}\subset \Sp^{d-1}$.  
\end{theorem}

\begin{proposition}\label{prop:non_inden}
Let $\{X^\mu_i\}_{1\leq i\leq n}$ and $\{X^{\nu}_i\}_{1\leq i\leq n}$ be two sets of points distributed under $F_\mu$ and $F_{\nu}$ respectively for $\mu,\nu>0$. Let $\tau$ be in $(0,1]$ and assume that $\mu>\nu$, then we have \[\P(\langle X^{\mu}_i,X^{\mu}_j\rangle\leq \tau)<\P(\langle X^\nu_i,X^\nu_j\rangle\leq \tau)\] for $i\neq j$. Moreover, there exists $\tau'\in (0,1]$ such that \[\P(\langle X^{\nu}_i,X^{\nu}_j\rangle\leq \tau)=\P(\langle X^\mu_i,X^\mu_j\rangle\leq \tau')\] for $i\neq j$.
\end{proposition}

\begin{remark}[Case $\tau=0$]
It is easy to see that in the case $\tau=0$ any measure with spherical symmetry we define the same $W$-random graph model. Intuitively speaking, the norm of the latent points is not used to decide the nodes connection, but only the fact that they belong to the same semisphere. In consequence, in the case $\tau=0$ we cannot recover the measure (nor distributional information about the latent points) from the adjacency matrix alone. 
\end{remark}

\begin{proposition}\label{prop:deg_threshold}
For the threshold graphon $W=W_g$ for $\tau\geq 0$ and $\{X_i\}_{1\leq i\leq n}\sim F_\nu$ for $\nu>-1/2$, we have for any $1\leq i\leq n$ \[\P\big(d_{W}(t_1N)\leq d_{W}(X_i)\leq d_W(t_2N)\big)=F_{Beta}(t_2^2)-F_{Beta}(t_1^2)\]
where $0<\tau<t_1<t_2$ and $F_{Beta}(\cdot)$ is the cumulative distribution function of $Beta~(~\frac{d}{2},~\nu+~\frac12)$. In addition, we have that
\[\P(d_W(X_i)=0)=F_{Beta}(\tau)\]
\end{proposition}


\subsection{Eigenvalue concentration}

The following theorem is a slight reformulation of the \cite[Cor.3.12]{BanVan} 
\begin{theorem}[Bandeira-Van Handel]\label{thm:bandeira_vanhandel}
Let $Y$ be a $n\times n$ symmetric random matrix whose entries $Y_{ij}$ are independent centered random variables. There exists a universal constant $C_0$ such that for $\alpha\in (0,1)$ \[\mathds{P}\Big(\|Y\|_{op}\geq 3\sqrt{2D_0}+C_0\sqrt{\log{n}/\alpha}\Big)\leq \alpha\]
where $D_0=\max_{0\leq i\leq n}\sum^n_{j=1}Y_{ij}(1-Y_{ij})$.
\end{theorem}

Using the previous theorem with $Y=\hat{T}_n-T_n$, which is centered and symmetric, we obtain the tail bound \[\mathds{P}\Big(\|\hat{T}_n-T_n\|_{op}\geq  \frac{3\sqrt{2D_0}}{n}+C_0\frac{\sqrt{\log{n}/\alpha}}{n}\Big)\leq \alpha\]

The next theorem is proven in \cite{Ara2} and gives a finite sample bound for the individual eigenvalues of $T_n$ with respect to the eigenvalues of the integral operator $T_W$. 
\begin{theorem}{\cite[Thm.2]{Ara2}}\label{thm:rel_con}
Let $(\Omega,\mu)$ be a probability space and $W:\Omega^2\rightarrow \mathds{R}$ be a $L^2(\Omega^2)$ kernel. Let $|\lambda_1|\geq|\lambda_2|\geq\cdots$ be the eigenvalues of integral operator $T_W$ and $\{\phi_i\}^\infty_{i=1}$ the a set of orthonormal eigenfunctions. Assume that $\|\phi_i\|_\infty=\O(i^s)$ and $|\lambda_i|=\O(i^{-\delta})$, where $\delta>2s+1$. Then we have with probability larger than $1-\alpha$\[|\lambda_i(T_n)-\lambda_i|\lesssim i^{-\delta+2s+1}n^{-\frac12}\quad \text{ for }1\leq i\leq n \] 
\end{theorem}
The following proposition gives a high probability bound for $\delta_2(\lambda(T_n),\lambda(T_W))$
\begin{proposition}\label{prop:delta_2}
Let $W$ be a graphon on $\B^d$ of the form $W(x,y)=f(\langle x,y\rangle)$ and $f\in S^p_{\gamma_{\nu}}([-1,1])$ for $p>2\nu-1+\frac{5d}2$, then we have with probability larger than $1-\alpha$ \[\delta_2(\lambda(T_n),\lambda(T_W))\lesssim_\alpha n^{-\frac12}\]
\end{proposition}

\begin{proof}
Define $d_l:=\dim{(\Y_n)}$. We will assume without loss of generality that $\{\lambda^\ast_k\}_{k\geq 0}$ is order decreasingly. Indeed, if $\sum_{l\geq 0}|\lambda^\ast_l| d_l<\infty$ holds, then $\{\lambda^\ast_k\}_{k\geq k_0}$ for some $k_0\in\N$ large enough, given that $d_l\asymp l^{d-1}$(with means that there exists $c,C>0$ such that $c l^{d-1}\leq d_l\leq C l^{d-1}$ for $l$ large enough). Define $l:\N\rightarrow\N$ to be the such that $\lambda_{i}=\lambda^\ast_{l(i)}$. 
From the relation $\sum^{l(i)-1}_{l=0}d_l\leq i\leq \sum^{l(i)}_{l=0}d_l$ we obtain \[(l(i)-1)^{d}\lesssim i\lesssim l(i)^{d}\] which implies that $l(i)=\O(i^{\frac{1}{d}})$.

Givent that $f\in S^p_{\gamma_{\nu}}([-1,1])$ the eigenvalues $\lambda^\ast_l$ satisfy $\sum_{l\geq 0} |\lambda^\ast_l|^2 d_l(1+\nu_l^{p})<\infty$, where $\nu_l=l(l+2\nu+d-1)$. This implies that $|\lambda^\ast_l|=\O(l^{-\delta^*})$ with $\delta^\ast=p+\frac{d}{2}+\varepsilon$ and $\varepsilon>0$. In consequence, we have $|\lambda_i|=\O(i^{-\delta})$, with $\delta:=\frac{p+\varepsilon}{d}+\frac{1}{2}$. By Lemma \ref{lem:inf_norm_estimates}, we have $\|p^2_{k,l}\|_{\infty}=\O( l^{2\nu+d-1})$, which given that $d_l\asymp l^{d-1}$, translate to $\|\phi_R\|_{\infty}=\O(R^{\frac{2\nu-1}{2d}+\frac12})$, for every $R\in\N$. Using Theorem \ref{thm:rel_con}, with $s=\frac{2\nu-1}{2d}+\frac12$ and $\delta=\frac{p+\varepsilon}{d}+\frac12$, we obtain 
 \[|\lambda_i(T_n)-\lambda_i|\lesssim_\alpha i^{-\delta+2s+1}n^{-1/2},\quad \text{ for }1\leq i\leq n
\]
with probability larger than $1-\alpha$.
If $p>2\nu-1+\frac{5d}2$ we the RHS of the previous inequality is summable, with respect to $i$, and the result follows. 

\end{proof}
For a graphon $W=f(\langle x,y\rangle)$ on $\B^d$, it is often useful to consider the sequence of eigenvalues of $T_W$ indexed with repetition. We will define as $\{\lambda'_l(f)\}_{\l\geq 0}$ the sequence of eigenvalues with repetitions, also ordered in the decreasing order for the absolute value. It is easy to see that each $\lambda^*_l(f)$ will appear $\dim{(\V_l)}$ times in $\{\lambda'_l(f)\}_{\l\geq 0}$ (if there exists $k$ such that $\lambda^*_l=\lambda^*_k$, then it will appear $\dim{(\V_l)}+\dim{(\V_k)}$ times). 

\begin{lemma}\label{lem:tprime}
If $W$ is graphon on $\B^d$ such that $W(x,y)=f(\langle x,y\rangle)$ and $f\in S^p_{\gamma_{\nu}}([-1,1])$ for $p>2\nu-1+\frac{5d}2$, with eigenvalues $\{\lambda'_l\}_{l\geq 0}$ (without repetition) and eigenfunctions $\{\phi_l\}_{l\geq 0}$. Define the $n\times n$ matrix with entries $(T'_n)_{ij}:=\frac1n \sum^{n-1}_{l=0}\lambda'_l\phi_l(X_i)\phi_l(X_j)$. Then we have \[\|T_n-T'_n\|_{op}=\O_\alpha(\frac{1}{\sqrt{n}})\] with probability larger than $1-\alpha$. 
\end{lemma}
\begin{proof}
We have that \[(T_n-T'_n)_{ij}=\frac1n \sum_{l\geq n}\lambda'_l\phi_l(X_i)\phi_l(X_j)\]
and by \cite[Thm.1]{Ara2} we have with probability larger than $1-\alpha$ \[\|T_n-T'_n\|_{op}=|\lambda'_n|+\O_\alpha(\frac{1}{\sqrt n})\]
On the other hand, given that $f\in S^p_{\gamma_{\nu}}([-1,1])$ for $p>2\nu-1+\frac{5d}2$, we have that $\lambda'_n\lesssim n^{-\frac{2\nu-1}d-2}$, hence the conclusion follows.
\end{proof}

\subsection{Eigenvectors concentration}

We will use the following version of the Davis-Kahan $sin$ $\theta$ theorem, which is stated and proved in \cite{YuWanSam}
\begin{theorem}[Davis-Kahan]\label{thm:davis_kahan}
Let $\Sigma$ and $\hat{\Sigma}$ be two symmetric $\mathbbm{R}^{n\times n}$ matrices with eigenvalues $\lambda_1\geq \lambda_2\geq \cdots \geq \lambda_n$ and  $\hat{\lambda}_1\geq \hat{\lambda}_2 \geq \cdots \hat{\lambda}_n $ respectively. For $1\leq r\leq s\leq n$ fixed, we assume that $\min{\{\lambda_{r-1}-\lambda_r, \lambda_s-\lambda_{s-1}\}}>0$ where $\lambda_0:=\infty$ and $\lambda_{n+1}=-\infty$. Let $d=s-r+1$ and $V$ and $\hat{V}$ two matrices in $\mathbbm{R}^{n\times d}$ with columns $(v_r,v_{r+1},\cdots,v_s)$ and $(\hat{v}_r,\hat{v}_{r+1},\cdots,\hat{v}_s)$ respectively, such that $\Sigma v_j=\lambda_j v_j$ and $\hat{\Sigma}\hat{v}_j=\hat{\lambda}_j \hat{v}_j$. Then there exists an orthogonal matrix $\hat{O}$ in $\mathbbm{R}^{d\times d}$ such that \begin{equation}\label{eq:thmDK}\|\hat{V}\hat{O}-V\|_{F}\leq \frac{2^{3/2}\min{\{\sqrt{d}\|\Sigma-\hat{\Sigma}\|_{op},\|\Sigma-\hat{\Sigma}\|_{F}\}}}{\min{\{\lambda_{r-1}-\lambda_{r},\lambda_s-\lambda_{s+1}\}}}\end{equation}
\end{theorem}

We recall that $\Phi_k=\frac1{\sqrt n}(\phi_k(X_1),\phi_k(X_2),\cdots,\phi_k(X_n))^T$. For $k,k'\in \N$ such that $k'>k$, we define $\Phi_{k:k'}$ as the matrix with columns $\Phi_k,\Phi_{k+1},\cdots,\Phi_{k'}$. Define $\V_1(k,k')=\|\sum^{k'}_{l=k}\phi^2_l\|_{\infty}$. 

\begin{proposition}\label{prop:phi_con}
We have with probability at least $1-\alpha$ \[\|\Phi_{k:k'}\Phi_{k:k'}^T-\I_{|k-k'|}\|_{op}\lesssim_\alpha \frac{\V_1(k,k')}{n}\wedge\sqrt{\frac{\V_1(k,k')}{n}} \]
\end{proposition}
\begin{proof}
The proof is identical to the proof of \cite[Prop.4]{Ara2}, which uses Matrix Bernstein inequality \cite[Thm.6.1]{Tropp}.
\end{proof}

\begin{lemma}\label{lem:projapprox}
Let $B$ a $n\times d$ matrix with full column rank. Then we have \[\| BB^T-B(B^TB)^{-1}B^T\|_F= \|\mathrm{Id}_d-B^TB\|_F\]
\end{lemma}
\begin{proof}
We have \begin{align*}
    \| BB^T-B(B^TB)^{-1}B^T\|_F&=\|B\big((B^TB)^{-1}-\mathrm{Id}_d\big)B^T\|_F\\
\end{align*}
and by definition of the Frobenious norm and cyclic property of the trace
\begin{align*}
    \|B\big((B^TB)^{-1}-\mathrm{Id}_d\big)B^T\|^2_F&=tr\big(B((B^TB)^{-1}-\mathrm{Id}_d)B^T B((B^TB)^{-1}-\mathrm{Id}_d)B^T\big)\\
    &=tr\big((\mathrm{Id}_d-B^TB)^2\big)\\
    &=\|\mathrm{Id}_d-B^TB\|^2_F
\end{align*}
\end{proof}

\section{Proofs}
Here we gather the proofs of the main results of the article. 

\begin{proof}[Proof of Lemma \ref{lem:dist_norm} ]
It is classic (see \cite[Sec.5]{Kelker}) that for a spherically symmetric distribution with density of the form $p(y)=g(\|y\|^2)$ where $y\in \B^d$, then the norm will have density $h(r)=\frac{2\pi^{d/2}}{\Gamma(d/2)}r^{d-1}g(r^2)$. The c.d.f for the radius of variable distributed following $F_\nu$ is proportional to $\int^t_0r^{d-1}(1-r^2)^{\nu-\frac12}dr$ using the change of variables $r\rightarrow r^2$ we obtain that the square of the norm have density $\int^t_0r^{\frac{d-1}{2}}(1-r)^{\nu-\frac12}dr$ where we recognize the density of a $Beta(\frac{d}2,\nu+\frac12)$. 
\end{proof}

\begin{proof}[Proof of Prop.\ref{prop:deg_threshold}]
Notice that in the case of the threshold graphon, the degree function $t\rightarrow d_W(tN)$ is increasing. Using this we have that 
\[\P(d_W(t_1N))\leq d_W(X_i)\leq d_W(t_2N)=\P(\|X_i\|\in [t_1,t_2])\]
Using the previous and Lemma \ref{lem:dist_norm}, the result follows. 
\end{proof}

\begin{proof}[Proof of Prop. \ref{prop:power-law_equiv}]

We will assume that $h'$ is a rational number. We choose $k,n'\in\N$ such that $h=k/n'$. We put $\theta_1=2.5$. We saw in Sec. that $\P(\frac{k-1}{n'}\leq \tilde{d}_f(x)\leq \frac k{n'})\lesssim (\frac k{n'})^{-\theta_1}$. We have following claim. \underline{Claim 1:} There exist a constant such that $\P(\tilde{d}_f(x)\geq h')\leq C {h'}^{1-\theta_1}$. We proof this claim. We have 

\begin{align*}
\P(\tilde{d}_f(x)\geq h')&\lesssim \sum^{n'}_{i=k+1}\P(\frac{i-1}{n'}\leq \tilde{d}_f(x)\leq \frac{i}{n'})\\
&\lesssim \frac1{(n')^{-\theta_1}}\sum^{n'}_{i=k+1}i^{-\theta_1}\\
&\lesssim \frac1{(n')^{-\theta_1}}\sum_{i>k}i^{-\theta_1}\\
&\lesssim (\frac{k}{n'})^{1-\theta_1}={h'}^{1-\theta_1}
\end{align*}

We have the following: \underline{Claim 2:} There exists a linear function $L_\alpha$ such that $\tilde{d}_f(x)\geq L_\alpha(d_f(x))$. we prove this claim. Define \begin{align*}J_\alpha&=\int_{\B^d\setminus \B^d(0,\sqrt \alpha)}\frac{1}{y_1^2}dF_\nu(y)\\ \kappa_\alpha&=\int_{\B^d(0,\sqrt \alpha)}dF_\nu(y)/J_\alpha \end{align*}
By definition, we have \[\frac{d_f(x)}{J_\alpha}=\frac1{J_\alpha}\int_{\B^d\setminus \B^d(0,\sqrt \alpha) }\frac{\alpha}{x^2_1y^2_1}\wedge 1dF_\nu(y)+\kappa_\alpha\]
By definition, $\tilde{d}_f(x)$ is larger than the first term in the RHS in the previous expression. This implies that \[\tilde d_f(x)\geq \frac1{J_\alpha}d_f(x)-\kappa_\alpha\]
Defining $L_\alpha(x)=x/{J_\alpha}-\kappa_\alpha$, the claim follows. 

With this, we have that there exist a contant $C>0$ such that
\begin{align*}
\P(d_f(x)\geq h)&\leq \P(L^{-1}_\alpha(\tilde d_f(x))\geq h)\\
&=\P(\tilde{d}_f(x)\geq L_\alpha(h))\\
&\leq C(L_\alpha(h))^{1-\theta_1}\\
&=C(\frac h{J_\alpha}-\kappa_\alpha)^{1-\theta_1}\\
&\leq \frac{C}{J_\alpha}(h-\kappa_\alpha J\alpha)^{1-\theta_1}
\end{align*}
which proves the proposition.
\end{proof}

\begin{proof}[Proof of Lemma \ref{lem:degree_fnct}]
For $t\leq \tau$ we have that $tN\in\B^{d}(0,\tau)$, which implies that $d_W(tN)=0$. The result for this case follows by noting that $I_0(a,b)=0$ for any $0\leq a,b\leq 1$. For $t>\tau$, call $h=(\tau/t)^2$ we have 
\begin{align*}
d_W(tN)&=\int_{\B^d}\mathbbm{1}_{\langle tN,y\rangle\geq \tau}(1-\|y\|^2)^{\nu-\frac12}\\
&\propto \int^1_h\int^{\sqrt{1-x^2_1}}_0r^{d-2}\big(1-x^2_1-r^2\big)^{\nu-\frac12}drdx_1\\
&\propto \int^1_h (1-x_1)^{\frac d2+\nu-1}dx_1\int^1_{0}\big(1-t)^{\nu-\frac12}t^{\frac{d-3}2}dt\\
&\propto \int^{1-h}_{0}x_1^{\frac{d}{2}+\nu-1}x_1^\frac12dx_1
\end{align*}
where we did a change a change of variables $t=\frac{r^2}{1-x^2_1}$ in the third line. The result follows from the fact the both quantities integrate $1$. 
\end{proof}

\begin{proof}[Proof of Lemma \ref{lem:inf_norm_estimates}]
From \cite{Dai}[Thm.11.1.12] we know that for each $n$ there exists $\psi_k\in\S^{d-1}$ such that $G^{\gamma_\nu}_n(\psi_k,x)$ is a basis of $\mathcal{Y}_n$. We take $p_{k,n}(x)=G^{\gamma_\nu}_n(\psi_k,x)$ for $1\leq k\leq \dim(\Y_n)$. From \cite[Eq.B.2.2]{Dai} and \cite[Thm.7.32.1]{Sze} we have $\|p_{k,n}\|_\infty\leq |G^{\gamma_\nu}_n(1)|\asymp n^{\gamma_\nu}$, because Gegenbauer polynomials are Jacobi polynomials with the repeated exponent parameter. For the second inequality, we use \eqref{eq:RKn2} and \eqref{eq:RKn1} to obtain 

\begin{align*}
\sum^{\dim{(\Y_n)}}_{k=1}p^2_{k,n}(x)&=c_\nu\frac{n+\gamma_\nu}{\gamma_\nu}\int_{-1}^1G_n^{\gamma_\nu}(\langle x,y\rangle+\sqrt{1-\|x\|^2}\sqrt{1-\|y\|^2}t)(1-t^2)^{\nu-1}dt\\
&\leq \|G^{\gamma_\nu}_n\|_{\infty}c_\nu\frac{n+\gamma_\nu}{\gamma_\nu}\int^1_{-1}(1-t^2)^{\nu-1}dt\\
&\lesssim |G^{\gamma_\nu}_n(1)|n c_\nu\int^1_{-1}(1-t^2)^{\nu-1}dt\\
&\lesssim n^{2\nu+d-1}c_\nu\int^1_{-1}(1-t^2)^{\nu-1}dt
\end{align*}
\end{proof}

\begin{proof}[Proof of Prop.\ref{prop:non_inden}]
For every $i$, we can write $X^\mu_i=R^\mu_iU_i$, where $R_i$ and $U_i$ are independent and $U_i$ is uniformly distributed on $\S^{d-1}$. Similarly, we can decompose $X^\nu_i=R^\nu_iV_i$, where $V_i$ is is uniformly distributed on $\S^{d-1}$. Given that $\mu<\nu$ we have that $(1-\|x\|^2)^{\nu-\frac12}<(1-\|x\|^2)^{\mu-\frac12}$, for $x\in \B^d$. This implies that $\P(R^\nu_i<\tau)>\P(R^\mu_i<\tau)$. Given that $\langle X^\mu_i,X^\mu_j\rangle=R^\mu_iR^\mu_j\langle U_i,U_j \rangle$ and $\langle X^\nu_i,X^\nu_j\rangle=R^\nu_iR^\nu_j\langle V_i,V_j \rangle$, and $\langle U_i,U_j \rangle\stackrel{D}{=}\langle V_i,V_j\rangle$ (they are equal in distribution), it is easy to see that \[\P(R^\mu_iR^\mu_j\langle U_i,U_j \rangle\leq\tau)<\P(R^\nu_iR^\nu_j\langle V_i,V_j \rangle\leq\tau)\]
To prove the second assertion, we see that by routine computations, the density of the inner product $\langle X^\mu_i,X^\mu_j\rangle$ is \[t\rightarrow \int^1_0\int^1_0((sr)^2-t^2)^{\frac d2-1}(1-r^2)^{\mu-1}(1-s^2)^{\mu-1}\mathbbm{1}_{sr>|t|}drds\]
\end{proof}

\begin{proof}[Proof of Prop.\ref{prop:conv_est_norm}]
Conditional to $X_i$, the random variable $d_G(X_i)$ is a sum of independent random variables, hence by the strong law of large number $d_G(X_i)\rightarrow d_W(X_i)$ almost surely. By the continuity of the function $t\rightarrow \frac{\tau}{\sqrt{1-I^{-1}(t)}}$, we deduce that \[\frac{\tau}{\sqrt{1-I^{-1}(2d_G(X_i)/(n-1))}}\rightarrow \frac{\tau}{\sqrt{1-I^{-1}(2d_W(X_i))}}=\|X_i\|\] in the almost sure sense. 
\end{proof}
\begin{proof}[Proof of Lemma \ref{lem:event_2}]
Invoke Theorem~\ref{thm:bandeira_vanhandel} with $Y=\hat{T}_n-T_n$, which has independent centered entries conditional to the latent points, to obtain with probability larger than $1-\alpha/2$
\begin{equation*}
    \|\hat{T}_n-T_n\|_{op}\lesssim_{\alpha}\frac{1}{\sqrt{n}}
\end{equation*} 
 because $D_0=\max_{0\leq i\leq n}\sum^n_{j=1}\Theta_{ij}(1-\Theta_{ij})$ is $\mathcal{O}(n\rho_n)$, by the definition of $\Theta$. Thus, there exists $n_0'\in \N$ such that for all $n\geq n_0$ we have $\|\hat{T}_n-T_n\|_{op}\leq \frac{{\Delta^*}^2}{2^{\frac{13}2}\sqrt d} $. It is easy to see that in this case $n_0'=\O({\Delta^*}^{-2}\sqrt{d}\log{2/\alpha})$.
 
From Theorem \ref{prop:delta_2} we have that, there exists $n_0''$ such that for $n\geq n_0''$ we have
\begin{equation}
\label{eq:delta2_event}
    \delta_2\Big(\lambda(T_n),\lambda(T_W)\Big)\lesssim_{\alpha}\frac{1}{\sqrt{n}} \leq\frac{ \Delta^\ast}4\,,
\end{equation}
with probability larger than $1-\alpha/2$. We see here that $n_0''=\O({\Delta^*}^{-1}\log{1/\alpha})$. Taking $n_0=\max{\{n_0',n_0''\}}$ we have that $\P(\mathcal E)\geq \alpha/2$, for $n\geq n_0$.
\end{proof}

\begin{proof}[Proof of Prop.\ref{prop:mainprop}]
First, notice that under $\mathcal{E}(W,n)$, we have \[\|T_n-\hat{T}_n\|_{op}\leq \frac{{\Delta^*}^2}{4\sqrt{d}}\leq \frac{\Delta^*}{4}\]
because $\Delta^*\leq1$, given that $0\leq W\leq 1$. We also have $\delta_2(\lambda(T_n),\lambda(T_W))\leq \frac{\Delta^*}{4}$. From that and the definition of $\delta_2(\cdot,\cdot)$ we deduce that there are at least $d$ eigenvalues of $T_n$ at distance less that $\Delta^*/4$ from $\lambda^*_1$ (given the multiplicity of $\lambda^*_1$). But each eigenvalue of $T_n$ is at distance at most $\Delta^*/4$ from an eigenvalue of $T_W$, and given that $dist(\lambda^*_1,\lambda(T_W)\setminus\{\lambda_1^*\})=\Delta^*$ we have that there are exactly $d$ eigenvalues of $T_n$ at distance at most $\Delta^*/4$ from $\lambda^*_1$. By the triangle inequality and $\|T_n-\hat{T}_n\|_{op}\leq \frac{\Delta^*}{4}$ we deduce that there exists a set of exactly $d$ eigenvalues of $\hat{T}_n$ at distance at most $\frac{\Delta^*}2$ from $\lambda^*_1$. 
\end{proof}

\begin{proof}[Proof of Thm. \ref{thm:main3}]
By Prop. \ref{prop:mainprop} we know that, under the eigengap condition, there is a cluster $\hat{\Lambda}_1$ of exactly $d$ eigenvalues of $\hat{T}_n$ and, another cluster $\Lambda_1$ of $d$ eigenvalues of $T_n$, such that all the elements in both clusters are at distance at most $\Delta^*/2$ from $\lambda^*_1$. We called $V$ (resp. $\hat{V}$) to the $n\times d$ matrix, where the columns are the eigenvectors of $T_n$ (resp. $\hat{T}_n$) associated with $\Lambda_1$(resp. $\hat{\Lambda}_1$). By Theorem~\ref{thm:bandeira_vanhandel} we have that \[ \|\hat{T}_n-T_n\|_{op}\lesssim_{\alpha}\frac{1}{\sqrt{n}}\] with probability larger than $1-\alpha$. By Thm. \ref{thm:davis_kahan} we have that with probability larger than $1-\alpha$\[\|VV^T-\hat{V}\hat{V}^T\|_{op}\lesssim_\alpha \frac{\sqrt d}{\Delta^*\sqrt n}\]
We will assume that $\{\lambda_l\}_{l\geq 0}$ is the sequence of eigenvalues of $T_W$ indexed with repetition. 
To prove the theorem will be sufficient to show that $\|\Phi_{l:l+d}\Phi_{l:l+d}^T~-~VV^T\|_{op}~=~\O(\frac1{\sqrt n})$ with probability at least $1-\alpha$. 

 We define the matrices $T'_n=\frac1n\sum^{n-1}_{l=0}\lambda_l\Phi_l\Phi^T_l$ and $\tilde{T}_n=\frac1n \sum^{n-1}_{l=0}\lambda_l\tilde{\Phi}_l\tilde{\Phi}^T_l$, where the vectors $\{\tilde{\Phi}_0,\cdots,\tilde{\Phi}_{n-1}\}$ are obtained from $\Phi_0,\cdots,\Phi_{n-1}$ by a Gram-Schmidt orthonormalization process. Observe that $\tilde{T}_n\tilde{\Phi}_l=\lambda_l\tilde{\Phi}_l$. We have the following claim.
 
\noindent \underline{Claim 1}: with probability larger than $1-\alpha$ we have $\|\tilde{T}_n-T'_n\|_{op}\lesssim_\alpha \O(\frac{1}{\sqrt n})$. 
 
 Assume this claim for the moment and define $\tilde{V}$ a matrix with columns $\tilde{\Phi}_{1},\cdots, \tilde{\Phi}_{d}$. By Lemma \ref{lem:tprime} we have $\|T_n-T'_n\|_{op}=\O_\alpha(\frac{1}{\sqrt{n}})$ with probability larger than $1-\alpha$, which implies that $\|\tilde{T}_n-T_n\|_{op}\lesssim_\alpha \frac1{\sqrt n}$ (by triangle inequality and Claim 1) and by Thm.\ref{thm:davis_kahan} we have that \[\|\tilde{V}\tilde{V}^T-VV^T\|_F\lesssim_{\alpha}\frac{\sqrt{d}}{ \Delta^*\sqrt n}\]
We will now prove Claim 1. Consider the notation 
\begin{align*} 
\tilde{V}_{d_l}&:=(\tilde{\Phi}_{d_l}|\tilde{\Phi}_{d_l+1}|\cdots|\tilde{\Phi}_{d_{l+1}})\\
\Phi_{d_l}&:=\Phi_{d_l:d_{l+1}}=(\Phi_{d_l}|\Phi_{d_l+1}|\cdots|\Phi_{d_{l+1}})
\end{align*}
Given that $\tilde{\Phi}$ is obtained by a Gram-Schmidt process from $\Phi$, we have that $\sp(\tilde{V}_{d_l})~=~\sp(\Phi_{d_l})$, where $\sp(A)$ is the linear span of the columns of matrix $A$. Hence the orthogonal projectors $\tilde{V}_{d_l}\tilde{V}^T_{d_l}$ and $\Phi_{d_l}(\Phi^T_{d_l}\Phi_{d_l})^{-1}\Phi^T_{d_l}$ are equal for every $l\leq l(n)$, where $l(n)$ is defined by $l(n)~=~\min\{l'\in\N:\sum^{l'}_{l=0}d_l\leq n\}$.  

On the other hand, we have that with probability at least $1-2\alpha$\begin{align}\|\Phi_{d_l}(\Phi^T_{d_l}\Phi_{d_l})^{-1}\Phi^T_{d_l}-\Phi_{d_l}\Phi^T_{d_l}\|_{F}&=\|\Phi^T_{d_l}\Phi_{d_l}-\I_{d_l}\|_{F}\nonumber\\ \label{eq:phi_ident}
&\lesssim_\alpha \sqrt{d_l\frac{\V_1(d_l, d_{l+1})}{n}}
\end{align}
where we used Lemma \ref{lem:projapprox} in the first step and Prop.\ref{prop:phi_con}, together with the bound $\|A\|_{F}~\leq~ \sqrt{d_l}\|A\|_{op}$ for a matrix of size $d_l$, in the last step. Notice that is possible to use Lemma \ref{lem:projapprox} because with probability $1-\alpha$ we have $\|\Phi^T_{d_l}\Phi_{d_l}-\I_{d_l}\|_{op}\lesssim_\alpha \frac{\V_1(d_l,d_{l+1})}{n}$, and $\frac{\V_1(d_l,d_{l+1})}{n}< 1$ for all $l\leq l(n)$. Hence, the event that $\Phi_{d_l}$ has full rank has probability at least $1-\alpha$. By Lemma \ref{lem:inf_norm_estimates} we have that $\V_1(d_l,d_{l+1})=\O(l^{2\nu+d-1})$, but given the assumption on the Sobolev regularity of $f$, we have $\sum^{l(n)}_{l=0}|\lambda^*_l|\sqrt{d_l\V_1(d_l,d_{l+1})}=\O(1)$, for all $l\leq l(n)$. Indeed, we have that $\sqrt{d_l\V_1(d_l,d_{l+1})}=\O(l^{2\nu-1+d})$ and $|\lambda_l^*|=\O(l^{-\delta^*})$, where $\delta^*>(2\nu-1+3d)$, which implies that $|\lambda^*_l|\sqrt{d_l\V_1(d_l,d_{l+1})}=\O(l^{-2d})$, which is summable. Given the spectral expansion of $\tilde T_n$ and $T'_n$, we have \[\|\tilde{T}_n-T'_n\|_{op}\leq \sum^{l(n)}_{l=0}|\lambda^*_l|\|\Phi_{d_l}(\Phi^T_{d_l}\Phi_{d_l})^{-1}\Phi^T_{d_l}-\Phi_{d_l}\Phi^T_{d_l}\|_{op}\] Bounding the operator norm by the Frobenius norm and \eqref{eq:phi_ident}, we have  \[\|\tilde{T}_n-T'_n\|_{op}\leq \sum^{l(n)}_{l=0}\lambda^*_l\sqrt{\frac{d_l\V_1(d_l,d_{l+1})}{n}}\lesssim_\alpha \frac{1}{\sqrt n}\]
This proves Claim 1. Notice that by \ref{eq:phi_ident} we have that $\|\Phi_{1}\Phi^T_{1}-\tilde{V}\tilde{V}^T\|_{op}\lesssim_{\alpha,d}\frac{1}{\sqrt n}$, which by triangular inequality gives that \[\|\hat{V}\hat{V}^T-\Phi_{1}\Phi^T_{1}\|_{F}\lesssim_{\alpha,d}\frac1{\Delta^*\sqrt n}\]
which concludes the proof.
\end{proof}

\end{document}